\documentclass{article}
\usepackage{paperstyle}
\usepackage{times}

\usepackage[utf8]{inputenc}
\usepackage[T1]{fontenc}
\usepackage{url}
\usepackage{amsfonts}
\usepackage{nicefrac}
\usepackage{microtype}
\usepackage{wrapfig}
\usepackage{mathtools}
\usepackage{amssymb}
\usepackage{enumitem}
\usepackage{amsthm}
\usepackage{amsmath}
\usepackage{subcaption}
\usepackage{xcolor,hyperref}
\usepackage[capitalise]{cleveref}
\usepackage{aliascnt}
\def\RevmarkMode{new}
\usepackage{revmark}
\usepackage{booktabs}
\usepackage{pgfplots}
\pgfplotsset{compat=1.18}
\usepgfplotslibrary{groupplots}
\usepgfplotslibrary{fillbetween}

\hypersetup{
  colorlinks,
  linkcolor={blue!50!black},
  citecolor={green!50!black},
  urlcolor={red!80!black}
}

\definecolor{gaussblue}{HTML}{1F4E79}
\definecolor{gaussorange}{HTML}{B5651D}
\definecolor{gaussred}{HTML}{8B1A1A}
\definecolor{gaussgrey}{HTML}{7A7A7A}
\definecolor{gaussbrown}{HTML}{C26A2A}

\pgfplotsset{
  theoryplot/.style={
    scale only axis,
    clip mode=individual,
    grid=major,
    grid style={gray!20},
    yticklabel style={/pgf/number format/1000 sep={}},
    tick label style={font=\scriptsize},
    label style={font=\small},
    title style={font=\small},
    legend style={font=\scriptsize, draw=none, fill=none},
    legend cell align=left,
    width=0.76\linewidth,
    height=3.4cm,
  }
}

\newcommand{\E}{\mathbb{E}}
\DeclarePairedDelimiter\abs{\lvert}{\rvert}
\DeclarePairedDelimiter\norm{\|}{\|}
\def\defeq{:=}
\def\rset{\mathbb{R}}
\def\eye{\operatorname{Id}}
\makeatletter
\def\p{\@ifstar{\p@star}{\p@normal}}
\def\p@normal#1{(#1)}
\def\p@star#1{\!\left(#1\right)}
\makeatother
\def\Xspace{\mathbb{R}^d}
\newcommand{\psAl}[2]{\left\langle#1,#2\right\rangle}
\providecommand{\D}{\mathop{}\!\mathrm{d}}
\newcommand{\lec}{\!\le\!}
\DeclareMathOperator{\arcsinh}{arcsinh}
\DeclareMathOperator{\Law}{Law}

\def\DEnoise#1{D_{\pi, #1}}
\def\OPnoise#1{N_{#1}}
\def\DEnoiseA#1{\replaced{\widetilde D^\theta_{\pi, #1}}{\widetilde D_{\pi, #1}}}
\def\PhiA#1{\widetilde\Phi^\theta_{#1}}
\def\PhiS#1{\widetilde\Phi_{#1}}
\def\CMstep#1{\mathcal L_{#1}}
\def\CMstepA#1{\replaced{\widetilde{\mathcal L}^\theta_{#1}}{\widetilde{\mathcal L}_{#1}}}

\newcommand{\equalcontribstar}{\textsuperscript{*}}

\newcommand{\mainfigheight}{2.44cm}

\newtheorem{theorem}{Theorem}[section]
\newtheoremstyle{compactassumption}
  {\topsep}{\topsep}{\itshape}{}{\bfseries}{.}{.5em}{\thmnumber{#2}}
\theoremstyle{compactassumption}
\newtheorem{assumption}{Assumption}

\theoremstyle{plain}
\newcommand{\newsharedtheorem}[2]{%
  \newaliascnt{#1}{theorem}\counterwithin*{#1}{section}%
  \newtheorem{#1}[#1]{#2}\aliascntresetthe{#1}}
\newsharedtheorem{remark}{Remark}
\newsharedtheorem{lemma}{Lemma}
\newsharedtheorem{proposition}{Proposition}
\newsharedtheorem{corollary}{Corollary}

\makeatletter
\@for\@tempa:=theorem,remark,lemma,proposition,corollary\do{%
  \expandafter\def\csname theH\@tempa\endcsname{\thetheorem}}
\makeatother

\crefname{appendix}{Appendix}{Appendices}
\Crefname{appendix}{Appendix}{Appendices}
\crefname{assumption}{Assumption}{Assumptions}
\crefformat{assumption}{#2#1#3}
\Crefformat{assumption}{#2#1#3}
\crefrangeformat{assumption}{#3#1#4--#5#2#6}
\Crefrangeformat{assumption}{#3#1#4--#5#2#6}
\crefmultiformat{assumption}{#2#1#3}{ and~#2#1#3}
  {, #2#1#3}{, and~#2#1#3}
\Crefmultiformat{assumption}{#2#1#3}{ and~#2#1#3}
  {, #2#1#3}{, and~#2#1#3}
\crefrangemultiformat{assumption}{#3#1#4--#5#2#6}
  { and~#3#1#4--#5#2#6}{, #3#1#4--#5#2#6}{, and~#3#1#4--#5#2#6}
\Crefrangemultiformat{assumption}{#3#1#4--#5#2#6}
  { and~#3#1#4--#5#2#6}{, #3#1#4--#5#2#6}{, and~#3#1#4--#5#2#6}

\title{Iterating Consistency Models: Stability, \\ Error Bounds and Noise Schedules}

\author{%
\normalfont
\begin{minipage}{\textwidth}
\centering
\textbf{Alessio Spagnoletti}$^{1}$\equalcontribstar \quad
\textbf{Abdul-Lateef Haji-Ali}$^{2}$\equalcontribstar
\\[0.35em]
\textbf{Andrés Almansa}$^{1}$ \quad
\textbf{Alain Oliviero Durmus}$^{3}$ \quad
\textbf{Eric Moulines}$^{4}$ \quad
\textbf{Marcelo Pereyra}$^{2}$
\\[0.35em]
$^{1}$Laboratoire MAP5, UMR 8145, Université Paris Cité, CNRS
\\
$^{2}$Heriot-Watt University, School of Mathematical and Computer Sciences \& Maxwell Institute for Mathematical Sciences
\\
$^{3}$École Polytechnique
\\
$^{4}$Laboratoire de Recherche d’EPITA, Le Kremlin-Bicêtre, France
\\
$^{5}$Computing and Mathematical Sciences Division, Mohamed bin Zayed University of Artificial Intelligence (MBZUAI), Abu Dhabi, United Arab Emirates
\\[0.35em]
\footnotesize
\equalcontribstar Joint primary authors
\end{minipage}%
}

\begin{document}
\maketitle
\begin{abstract}
Consistency models (CMs) have become a leading approach for generating
high-quality samples in few steps. However, adding steps can improve or degrade
sample quality in ways that are highly sensitive to the schedule and that
existing theory does not fully explain. To provide accuracy guarantees and guide CM sampler design, we analyze multistep CM sampling as a composition of noising and approximate denoising operators. Under explicit, verifiable stability
assumptions, we derive a non-asymptotic error bound that separates
contraction of the initialization error from accumulation of approximation
error. The bound assigns distinct roles to the schedule: large early noise
levels drive contraction, while small late noise levels control the
residual bias. As a corollary, we obtain explicit constants for
strongly log-concave and semi-log-concave targets. We further establish a complementary guarantee whose assumptions, one-step accuracy and stability, can be estimated for a given trained model.  Experiments show that the
contraction and approximation profiles entering our bounds can be reliably
measured and closely match the predicted functional forms. Together, these
results provide a meaningful convergence theory for multi-step CMs and a
practical route to sampler design.
\end{abstract}

\section{Introduction}
Diffusion models (DMs)~\citep{sohl2015deep,song2019generative,song2021sde,ho2020ddpm} have revolutionized generative machine learning, achieving remarkable results in image~\citep{ramesh2022hierarchical,rombach2022ldm,zhang2023adding}, video~\citep{ho2022imagenvideohighdefinition,blattmann2023stable,wan2025}, 3D~\citep{liu2023zero,poole2022dreamfusion,wang2023prolificdreamer}, and audio generation~\citep{evans2024fast,liu2023audioldm}. Early DMs required hundreds to thousands of neural function evaluations (NFEs) per sample. However, distillation substantially reduces this cost, such as direct~\citep{Luhman2021KnowledgeDI,Zheng2022FastSO}, adversarial~\citep{Wang2022DiffusionGANTG,sauer2024adversarial}, or variational score distillation~\citep{wang2023prolificdreamer,yin2024one,yin2024improved,Luo2023DiffInstructAU,Salimans2024MultistepDO}. Within this context, consistency models (CMs)~\citep{song2023consistency,songdhariwal2023improved,kim2023consistency} have emerged as particularly effective few-step DMs. A CM learns a \emph{consistency function} that maps noisy points along a probability-flow ODE trajectory directly to clean data. It can be distilled from a pre-trained DM\added{ (consistency distillation)} or trained from scratch \replaced{(consistency training)}{via consistency training}.

Training CMs is nonetheless delicate. Discrete-time training introduces a discretization error that degrades sample quality, whereas naive continuous-time training is notoriously unstable~\citep{song2023consistency,songdhariwal2023improved,geng2024consistency}. sCM~\citep{lu2025simplifyingstabilizingscalingcontinuoustime} addressed this with TrigFlow, a formulation that greatly improves stability. A closely related line of work~\citep{geng2025mean,sabour2025align,zhou2025inductive} learns \emph{flow maps}~\citep{boffi2024flow}, which transport samples between any two points of the trajectory and include CMs as an important special case. Methods such as IMM~\citep{zhou2025inductive}, MeanFlow~\citep{geng2025mean}, and SCoT~\citep{wu2025traflow} make training more stable still by adding constraints to the learning objective.

In principle, CMs can generate a sample with a single network evaluation, but practical samplers often apply the model several times to improve sample quality. Each iteration re-noises the current sample and maps it back to the target with the learned consistency function. Multistep CMs, however, are brittle: additional steps do not always help, sample quality is highly sensitive to the noise schedule, and stopping early sometimes yields better samples than continuing. Existing theory for multistep samplers is limited and offers little guidance to practitioners designing or deploying CMs.

Intuitively, multistep CM sampling trades off two competing effects. On the one hand, each additional step can reduce the discrepancy between the generated and target distributions. On the other hand, each step introduces a new error from model bias. The key question is thus whether the contraction gained at each step outweighs the error it introduces. This paper makes that balance explicit.

To this end, we study multistep samplers that alternately apply stochastic noising and deterministic denoising operators to probability distributions. Our main contribution is a non-asymptotic convergence theory for multistep CMs that explains how the schedule governs sample quality and provides a principled basis for sampler design. Specifically:
\begin{itemize}[leftmargin=*,itemsep=2pt,topsep=3pt]
  \item \textbf{Error bound and the role of the schedule.} We derive a non-asymptotic error bound that separates contraction of the initialization error from accumulated approximation error. The schedule plays two roles: large early steps drive contraction, while small late steps control the residual bias. For strongly log-concave and semi-log-concave targets, we obtain explicit constants.
  \item \textbf{Guarantees for trained CMs.} We provide a complementary guarantee whose assumptions, one-step accuracy and stability of the learned map, can be estimated for any trained model.
  \item \textbf{Empirical validation.} Experiments show that the contraction
    and approximation profiles in our bounds can be measured reliably and
    closely match the functional forms predicted by our theory.
\end{itemize}

\paragraph{Relation to existing work.}
Convergence theory for CM sampling is recent. \citet{Lyu2024SamplingEasyCM}
analyze one-step and multistep sampling under Ornstein--Uhlenbeck noising.
Assuming that the learned map is Lipschitz and that the score and
consistency training errors are small, they bound the Wasserstein error of
the sampler and show that repeating a suitable noise level reduces this
error geometrically, down to a plateau set by the training errors.
\citet{chen2025convergenceconsistencymodelmultistep} derive multistep
guarantees for bounded targets via a Kullback--Leibler recursion, and other
work addresses statistical accuracy and consistency
training~\citep{dou2024theory,li2025consistencytraining}.

We build on the error recursion of \citet{Lyu2024SamplingEasyCM} but establish its ingredients under different assumptions. First, as they suggest, we place the stability condition required for contraction on the exact denoiser rather than on the learned network, and show that it holds under natural conditions on the target, such as semi-log-concavity. Second, we allow general Gaussian noising processes and arbitrary schedules. Third, we adopt a more realistic error model in which the score error may diverge as the noise level tends to zero~\cite[Assumption~A3]{debortoli2022manifold}, where noising no longer smooths the target, rather than remaining uniformly small. We make the correspondence between the two sets of results precise after \cref{thm:main}, in \cref{sec:relaxed-noise-CM}, and in \cref{sec:app_compars}.

\section{Background}\label{sec:background}
\paragraph{Notation.} We work on $\mathbb{R}^d$ with the Euclidean norm $\|\cdot\|$, and
denote by $\mathcal{P}_2(\Xspace)$ the set of Borel probability measures on
$\Xspace$ with finite second moment. For a measurable map $F$ and
$\rho \in \mathcal{P}_2(\Xspace)$, the pushforward $F_\#\rho$ is the law of
$F(X)$ with $X \sim \rho$. We identify a measure with its Lebesgue density
whenever the latter exists. For $r > 0$, we define the Gaussian smoothing
$\rho_r \coloneqq \kappa_r * \rho$, where
$\kappa_r(x) \coloneqq (2\pi r^2)^{-d/2} \exp\!\big(-\|x\|^2/(2r^2)\big)$;
equivalently, $\rho_r$ is the law of $X + rZ$ with $X \sim \rho$ and
$Z \sim \mathcal{N}(0, \mathbb{I})$ independent. The $2$-Wasserstein distance
between $\rho_1, \rho_2 \in \mathcal{P}_2(\Xspace)$ is
$W_2^2(\rho_1,\rho_2):=\inf_{\Gamma\in\Pi(\rho_1,\rho_2)}
\int\|x-y\|^2\,\Gamma(\D x,\D y)$,
where $\Pi(\rho_1, \rho_2)$ is the set of couplings of $\rho_1$ and $\rho_2$. A dot denotes differentiation in time, e.g.\
$\dot{x}_t = \D x_t / \D t$.

\paragraph{Diffusion models.} Diffusion models (DMs) draw approximate samples from a target distribution $\pi\in\mathcal P_2(\Xspace)$ by reversing an $\mathbb{R}^d$-valued noising process $X_t=\alpha_tX_0+\sigma_t Z$, where $X_0\sim\pi$ and
$Z\sim\mathcal N(0, \mathbb{I})$ are independent. The schedules $\alpha$ and $\sigma$ are non-negative and differentiable, with $\alpha_0 = 1$, $\sigma_0 = 0$ and $\alpha_t>0$ on the sampling interval. The noise-to-signal
ratio $\gamma_t \coloneqq \sigma_t / \alpha_t$ is non-decreasing, with $\gamma_t>0$ for $t>0$. For $t>0$, the law of $X_t$ has density
$\varpi_t(x)=\alpha_t^{-d}\pi_{\gamma_t}(x/\alpha_t)$. The reversal is typically implemented via the probability-flow ODE~\citep{song2021sde}
\begin{equation}\label{eq:grad-flow}
  \dot{x}_t = a_t x_t - \tfrac{1}{2} b_t^2 \nabla \log \varpi_t(x_t),
  \qquad
  a_t = \dot{\alpha}_t / \alpha_t, \quad
  b_t^2 = 2 \alpha_t \sigma_t \dot{\gamma}_t,
\end{equation}
whose solutions preserve the marginals $(\varpi_t)_{t > 0}$, when the flow is solved backwards from $X_t$. Let $\Phi_{t \to 0}$ denote the resulting map from time $t$ to time $0$, note that $(\Phi_{t \to 0})_\# \varpi_t = \pi$. Although easy to train, early DMs approximate the score
$\nabla \log \varpi_t$ rather than the solution of~\eqref{eq:grad-flow}, and
must therefore integrate~\eqref{eq:grad-flow} numerically to $0$,
at a cost of hundreds of neural function evaluations (NFEs) per sample. This drawback has been addressed through so-called distilled DMs, notably consistency models.

\paragraph{Consistency models.} Consistency models (CMs) approximate the flow map $\Phi_{t \to 0}$, predicting the endpoint at time $0$ of a state at time $t$ in a single NFE. They are trained either by distilling a pre-trained DM or directly from data via consistency training~\citep{song2023consistency,songdhariwal2023improved,boffi2025selfdistill}, using two properties of the flow map: self-consistency, $\Phi_{t \to 0}(x_t) = \Phi_{s \to 0}(x_s)$ for states on the same trajectory of~\eqref{eq:grad-flow}, and the boundary condition $\Phi_{0 \to 0}(x) = x$. The choice of $\alpha_t$ and $\sigma_t$ also affects training stability. Two common choices are the variance-exploding (VE) schedule, $\alpha_t = 1$ and $\sigma_t = t$ for $t \in [0, \infty)$~\citep{song2023consistency,songdhariwal2023improved}, and the TrigFlow schedule, $\alpha_t = \cos t$ and $\sigma_t = \sin t$ for $t \in [0, \pi/2)$~\citep{geng2024consistency,lu2025simplifyingstabilizingscalingcontinuoustime}.

Since the learned CM map $\replaced{\PhiA{t \to 0}}{\tilde{\Phi}_{t \to 0}}$\added{, a neural network with weights $\theta$,} only approximates
$\Phi_{t \to 0}$, a single CM step can yield insufficiently accurate
samples. Multistep sampling alternates forward noising with reverse
transport. \citet{song2023consistency} and \citet{heek2024multistep} show
that these iterations sometimes improve sample quality.
Namely, starting from an initialisation $X_0^{(0)}$, iteration $k \ge 1$ computes
\begin{equation}\label{eq:iterative-cm-sampler}
  X_{t_k}^{(k)} = \alpha_{t_k} X_0^{(k-1)} + \sigma_{t_k} Z^{(k)},
  \qquad
  X_0^{(k)} = \replaced{\PhiA{t_k \to 0}}{\tilde{\Phi}_{t_k \to 0}}\big(X_{t_k}^{(k)}\big),
\end{equation}
where $X_0^{(0)}, Z^{(1)}, \dots, Z^{(\replaced{n}{K})}$ are independent, with
$Z^{(k)} \sim \mathcal{N}(0, \eye)$, and $t_1 > \dots > t_{\replaced{n}{K}} > 0$.
In practice, $2$--$8$ iterations often suffice for excellent samples, yet more iterations can improve or degrade quality in ways that depend
sensitively on the schedule and that existing theory does not fully explain.

We herein close this gap with non-asymptotic bounds on the distance between the law
of $X_0^{(k)}$ and $\pi$, explicit in the initialization, schedule and time
steps. The bounds separate contraction of the initialization error, driven
by large early noise levels, from accumulation of approximation error,
controlled by small late noise levels. Crucially, their assumptions can be
estimated directly from a pre-trained model, so the bounds can be evaluated
in practice rather than serving only as qualitative guidance. For strongly
log-concave and semi-log-concave targets, all constants are explicit.

\section{Main results}\label{sec:main_results}

\subsection{Convergence of multistep consistency models}\label{sec:gen-noise-CM}

Let $\pi\in\mathcal P_2(\Xspace)$ be the target. We describe the sampler
\eqref{eq:iterative-cm-sampler} through its action on probability distributions. For $t>0$, the
noising operator $\OPnoise{t}:\mathcal P_2(\Xspace)\to\mathcal P_2(\Xspace)$
maps the law of $X$ to the law of $\alpha_tX+\sigma_tZ$, with
$Z\sim\mathcal N(0,\eye)$ independent of $X$, so that
$\OPnoise{t}(\pi)=\varpi_t$; the exact denoising operator is the
pushforward by the flow map of \eqref{eq:grad-flow},
$\DEnoise{t}(\rho)=(\Phi_{t\to0})_\#\rho$; and the learned denoising
operator is $\DEnoiseA{t}(\rho)=(\replaced{\PhiA{t\to0}}{\widetilde\Phi_{t\to0}})_\#\rho$. The
exact and learned sampling steps are
$\CMstep{t}=\DEnoise{t}\circ\OPnoise{t}$ and
$\CMstepA{t}=\DEnoiseA{t}\circ\OPnoise{t}$, so that the law $\rho_k$ of
\deleted{the iterate }$\replaced{X_0^{(k)}}{U_k:=X_0^{(k)}}$ in \eqref{eq:iterative-cm-sampler} satisfies
$\rho_k=\CMstepA{t_k}(\rho_{k-1})$\replaced{, $k\ge1$, where $\rho_0$ is the law of $X_0^{(0)}$}{ with $\rho_0=\Law(U_0)$}. The results of
this section use only the properties of these operators stated below, and
therefore hold for any maps
$\DEnoise{t},\DEnoiseA{t}:\mathcal P_2(\Xspace)\to\mathcal P_2(\Xspace)$
with $\CMstepA{t}$ preserving $\mathcal P_2(\Xspace)$, which is the case
when $\replaced{\PhiA{t\to0}}{\widetilde\Phi_{t\to0}}$ has at most linear growth. Fix finite
nonnegative profiles $R_t,L_t,E_t$ for noising stability, denoising
stability and approximation error, respectively.

\begin{assumption}\label{ass:oneshot}
For every $t>0$, $\CMstep{t}(\pi) = \pi$.
\end{assumption}

\begin{assumption}\label{ass:noising}
For all $t>0$ and $\rho_1,\rho_2\in\mathcal P_2(\Xspace)$,
$W_2\big(\OPnoise{t}(\rho_1),\,\OPnoise{t}(\rho_2)\big)
\;\le\; R_t\,W_2(\rho_1,\rho_2)$.

\end{assumption}

\begin{assumption}\label{ass:denoise}
For all $t>0$ and $\rho_1,\rho_2\in\mathcal P_2(\Xspace)$,
$ W_2\big(\DEnoise{t}(\rho_1),\,\DEnoise{t}(\rho_2)\big)
\;\le\; L_t\,W_2(\rho_1,\rho_2)$.

\end{assumption}

\begin{assumption}\label{ass:oneshot-approximation}
For every $t>0$ and $\rho\in\mathcal P_2(\Xspace)$, $W_2\big(\CMstep{t}(\rho),\,\CMstepA{t}(\rho)\big)\le E_t$.
\end{assumption}

\begin{proposition}[Finite-horizon sampling error]\label{thm:main}
  Let \cref{ass:oneshot,ass:noising,ass:denoise,ass:oneshot-approximation}
  hold and set $q_t:=R_tL_t$. Fix $T>0$ and an integer $n\ge1$. Let
  $t_1,\ldots,t_n\in(0,T]$ be an arbitrary schedule and
  $\rho_0\in\mathcal P_2(\Xspace)$ an initialization. Define
  $\rho_k:=\CMstepA{t_k}(\rho_{k-1})$ for $k=1,\ldots,n$ and
  $Q_{k,j}:=\prod_{i=k}^{j}q_{t_i}$ for $1\lec k\lec j\lec n$, with
  $Q_{j+1,j}=1$. Then\added{ the error after $n$ steps satisfies}
  \begin{equation}\label{eq:multi-stage-cm-err-bound}
    W_2(\rho_n,\pi) \le  Q_{1,n} W_2(\rho_{0},\pi)
                   + \sum_{k=1}^{n} Q_{k+1,n} E_{t_k} \,.
  \end{equation}
  If, additionally, the schedule satisfies the balance condition
  $q_{t_k}E_{t_{k-1}}\le\xi_n E_{t_k}$ for $k=2,\ldots,n$ and some
  $\xi_n\in(0,1)$, then
  \begin{equation}\label{eq:cm-geometric-error-bound}
    \sum_{k=1}^n Q_{k+1,n}E_{t_k}
    \le E_{t_n}\frac{1-\xi_n^n}{1-\xi_n}
    \le\frac{E_{t_n}}{1-\xi_n}.
  \end{equation}
  The schedule may repeat noise levels and need not be decreasing.
  The balance condition also covers zero error bounds;
  for $n=1$ it is empty and the sum equals $E_{t_1}$.
\end{proposition}
The proof is in \cref{sec:proof-gause-noise-CM}.

The proposition uses a uniform error bound for the learned step and stability
of the exact step; it does not require the learned step to be contractive. The result alone implies neither an optimal
finite number of steps nor deterioration when more steps are used.

\paragraph{Repeating a fixed noise level.}
For a fixed $t>0$ with $0\lec q_t<1$, let
$\rho_n=(\CMstepA{t})^n(\rho_0)$. The geometric sum in
\eqref{eq:multi-stage-cm-err-bound} gives
$W_2(\rho_n,\pi)\le q_t^n W_2(\rho_0,\pi)
  +E_t(1-q_t^n)/(1-q_t)$ for $n\ge1$.
Iteration attenuates the initial-distance term, while the accumulated-error
bound approaches $E_t/(1-q_t)$ (for their factors, this is the
constant-level bound of \citet[Corollary~3.11]{Lyu2024SamplingEasyCM},
minimized over the level). Repeating a noisy step can therefore lower the
certificate for a large initial error, while reducing the final bound may
require a level with a smaller ratio $E_t/(1-q_t)$.

\paragraph{Adding one step.}
Given a certificate $\bar d \ge W_2(\rho, \pi)$, a step at level $t$ yields $q_t \bar d + E_t$, which improves on $\bar d$ exactly when $E_t < (1-q_t)\bar d$, i.e., when $\bar d$ exceeds the plateau $E_t/(1-q_t)$; no improvement is possible for $q_t \ge 1$. This is the case distinction in the proof of \citet[Corollary~3.11]{Lyu2024SamplingEasyCM}, here valid at every level with $q_t < 1$ and for both placements of the hypotheses (\cref{prop:learned-recursion}). An extra step pays off only if its attenuation outweighs its error budget, and lowering the noise level reduces $E_t$ but weakens attenuation (\cref{sec:gauss-noise-CM}). These statements compare certificates, not actual errors.

\citet[Theorem~2 and Section~3.3]{chen2025convergenceconsistencymodelmultistep} use a related decomposition: appending a lower-noise step increases their accumulated-error term but reduces the final denoising-error term, a trade-off they exploit to choose two-step OU and geometric VE schedules. Their recursion, however, propagates the KL divergence between noisy marginals with coefficient one, while our factors $q_t$ quantify when later steps attenuate earlier errors. \Cref{rem:chen,sec:app_compars} compare the assumptions and bounds.

\subsection{Noise-dependent stability and approximation error}\label{sec:gauss-noise-CM}
We specialize
\cref{ass:oneshot,ass:noising,ass:denoise,ass:oneshot-approximation} to the
operators of \cref{sec:gen-noise-CM} when the learned map is an
approximate flow of the probability-flow ODE. Replacing \(\nabla \log \varpi_{t}\p{\cdot}\)
in \eqref{eq:grad-flow} by an \replaced{approximation}{estimate}
\(\replaced{s_\theta(\cdot,t)}{s_{t}\p{\cdot}}\) \deleted{learned from data} gives the
\replaced{approximate-score}{learned-score} ODE $\dot
x_t=a_tx_t-\tfrac12 b_t^2\replaced{s_\theta(x_t,t)}{s_t(x_t)}$.  In this section
\replaced{the learned map $\PhiA{t\to0}$ is}{$\widetilde{\Phi}_{t\to 0}$ denotes} its flow map from time $t$ to $0$, so that
$\DEnoiseA{t}(\rho)=(\replaced{\PhiA{t\to0}}{\widetilde\Phi_{t\to0}})_\#\rho$ is the denoiser induced by the exact
flow of the \replaced{approximate}{learned} score.
\added{The score is not learned directly. If $\PhiA{t\to0}$ is a
  diffeomorphism for every $t$, depends differentiably on $t$ and is the
  identity at $t=0$, the curves $t\mapsto(\PhiA{t\to0})^{-1}(x_0)$ have a
  velocity field $V_t$, and $s_\theta(x,t):=2(a_tx-V_t(x))/b_t^2$ is the score
  implied by the learned map wherever $b_t>0$.}

With these definitions, \cref{ass:oneshot} holds by construction: $\OPnoise{t}(\pi)=\varpi_t$ and, since the probability-flow ODE \eqref{eq:grad-flow} transports $\varpi_t$ to $\varpi_0=\pi$~\citep{song2021sde}, $\CMstep{t}(\pi)=(\Phi_{t\to 0})_{\#}\varpi_t=\pi$ for all $t>0$.

We now verify \cref{ass:noising,ass:denoise,ass:oneshot-approximation}
for these operators, starting with the stability of $\OPnoise{t}$.

\begin{lemma}\label{thm:lipschitz-noiser}
Let $\alpha_t$ be the signal coefficient of the noising process of
\cref{sec:background}, so that $\OPnoise{t}$ maps the law of $X$ to that
of $\alpha_tX+\sigma_tZ$. Then $\OPnoise{t}$ is $\alpha_t$-Lipschitz for
$W_2$, and $\alpha_t$ is its exact Lipschitz constant; that is,
\cref{ass:noising} holds with $R_t=\alpha_t$.
\end{lemma}
The proof is in \cref{sec:proof-gause-noise-CM}; we now turn to
\cref{ass:denoise,ass:oneshot-approximation}.

\begin{theorem}\label{theo:lip-error-ft}
  Let $\pi\in\mathcal P_2(\Xspace)$ and $f_r=-\log\pi_r$, where
  $\pi_r=\kappa_r*\pi$. Assume $\alpha$ and $\gamma=\sigma/\alpha$ are
  absolutely continuous on every sampling interval $[0,t]$, with
  $\alpha_0=1$, $\gamma_0=0$,
  $\alpha_t,\gamma_t>0$ for $t>0$, and $\gamma$ nondecreasing.
  Suppose \eqref{eq:grad-flow} and the \replaced{approximate-score}{learned-score} ODE, with
  $\replaced{s_\theta(x,t)}{s_t(x)}=\alpha_t^{-1}\replaced{\widetilde s_\theta(x/\alpha_t,\gamma_t)}{\widetilde s_{\gamma_t}(x/\alpha_t)}$, admit Borel backward flow maps whose trajectories, for every terminal state, are locally absolutely continuous on $(0, t]$ and extend continuously to $[0, t]$. Assume $\nabla^2 f_r \succeq \mu_r \eye$ for all $r > 0$, and define
  \begin{equation}\label{eq:approx-assumption}
    e_r:=\sup_{x\in\Xspace}
      \norm{\replaced{\widetilde s_\theta(x,r)}{\widetilde s_r(x)}-\nabla\log\pi_r(x)}<\infty.
  \end{equation}
  Suppose $\mu_r,e_r$ are measurable and
  $r\mu_r^-,re_r\in L^1_{\mathrm{loc}}([0,\infty))$, where
  $\mu_r^-:=\max\{-\mu_r,0\}$. Then the pushforward operators
  $\DEnoise{t}$ preserve $\mathcal P_2(\Xspace)$, and
  \cref{ass:denoise,ass:oneshot-approximation} hold with
\begin{equation}\label{eq:general_Lt_Et}
\begin{split}
L_{t} &= \frac{1}{\alpha_{t}} \, \exp\p*{- \int_{0}^{\gamma_{t}} u \, \mu_{u} \D u}\,,\quad
 E_t = \int_{0}^{\gamma_{t}}
                r \, e_r \,  \exp\p*{- \int_{0}^{r} u \mu_{u} \D u}
                \D r\,.
\end{split}
\end{equation}
  The integrals from zero are improper integrals: since only the negative
  part of $u\mu_u$ is assumed integrable near zero,
  $\int_0^r u\mu_u\D u$ may equal $+\infty$, in which case the
  exponential is set to zero.
\end{theorem}
The proof is in \cref{sec:proof-gause-noise-CM}.

Strong log-concavity gives an explicit family of stability profiles.
We say that $\pi$ is $m$-strongly log-concave, for $m\ge0$, if it has a
positive Lebesgue density and $f=-\log\pi$ is $m$-strongly convex.

\begin{corollary}[Strongly log-concave case]\label{coro:strong-log-concave}
Assume the conditions of \cref{theo:lip-error-ft} hold and that there exist \(M_{0}, M_{1} \geq 0\) such that the score approximation error \eqref{eq:approx-assumption} satisfies
  \(
    e_{r}  \leq M_0 + M_1 /r\,,
\) for all \(r>0\). Assume further that $\pi$ is \(m\)-strongly log-concave with \(m>0\) and that $f=-\log \pi$ is \(C^2\). Then, \cref{ass:denoise,ass:oneshot-approximation} hold with
$L_t=\alpha_t^{-1}(1+m\gamma_t^2)^{-\nicefrac{1}{2}}$ and
$E_t \leq \frac{M_{0}}{m}
        \p*{\sqrt{1+m \gamma_{t}^{2}} - 1}
        +
        \frac{M_{1}}{\sqrt{m}}
        \, \arcsinh\p*{\sqrt{m} \gamma_{t} }$.
\end{corollary}
The proof is in \cref{sec:proof-gause-noise-CM}.

\begin{remark}[Variance-exploding \& TrigFlow CMs]\label{rem:VE-TrigFlow}
Under the assumptions of \cref{coro:strong-log-concave}, define
$B_m(r):=\frac{M_0}{m}\p*{\sqrt{1+mr^2}-1}
    +\frac{M_1}{\sqrt m}\arcsinh\p{\sqrt m\,r}$.
\Cref{thm:lipschitz-noiser} and the corollary give $R_t=1$,
$L_t=(1+mt^2)^{-\nicefrac{1}{2}}$ and $E_t=B_m(t)$ for VE\deleted{ noising} ($\alpha_t=1$,
$\sigma_t=t$), and $R_t=\cos t$, $L_t=(\cos^2t+m\sin^2t)^{-\nicefrac{1}{2}}$ and
$E_t=B_m(\tan t)$ for TrigFlow ($\alpha_t=\cos t$, $\sigma_t=\sin t$,
$0<t<\pi/2$); in both cases \deleted{the exact step }$\CMstep{t}$ is
$q_t$-contractive in $W_2$ with $q_t=\deleted{R_tL_t=}(1+m\gamma_t^2)^{-\nicefrac{1}{2}}<1$ \replaced{for
all admissible $t>0$}{at
every positive admissible time}.
\end{remark}

\paragraph{Beyond strong log-concavity.} We now apply \cref{theo:lip-error-ft} to nonconvex potentials, extending the strongly log-concave case of
\cref{coro:strong-log-concave} at the cost of weaker contraction.
For $\ell\ge0$ define
\begin{equation}\label{eq:psi-def}
  \psi_\ell(u):=2\sqrt\ell\,u\tanh(\sqrt\ell\,u/2),\qquad u\ge0,
\end{equation}
and consider, for $m>0$, the relaxed curvature condition
\begin{equation}\label{eq:weakly-convex}
  \psAl{\nabla f(x)-\nabla f(y)}{x-y}
  \ge m\norm{x-y}^2-\psi_\ell(\norm{x-y}),\qquad x,y\in\Xspace.
\end{equation}
For $\ell = 0$ this is $m$-strong convexity. For $\ell > 0$, the defect $\psi_\ell$ is quadratic near zero and grows linearly at infinity, so the quadratic term set by $m$ dominates at large separations while $\ell$ controls the curvature defect near zero. With $m$ as in \eqref{eq:weakly-convex} and score-error constants $M_0, M_1 \ge 0$, let $B_m$ be as in \cref{rem:VE-TrigFlow} and set $A_{m,\ell}(r) := \exp\{\ell r^2 / [2(1 + m r^2)]\}$ for $r \ge 0$.

\begin{theorem}[Stability]\label{thm:weak-curvature}
Assume the parametrization and backward-flow hypotheses of
\cref{theo:lip-error-ft}. Let $\pi\in\mathcal P_2(\Xspace)$ have a positive
density with $f=-\log\pi\in C^1$ satisfying \eqref{eq:weakly-convex}.
Assume the score error $e_r$ in \eqref{eq:approx-assumption} is measurable
and satisfies $e_r\le M_0+M_1/r$ for every $r>0$.
Then \cref{ass:noising,ass:denoise,ass:oneshot-approximation} hold with
\begin{equation}\label{eq:weak-curvature-profiles}
  R_t=\alpha_t,\qquad
  L_t=\frac{A_{m,\ell}(\gamma_t)}{\alpha_t\sqrt{1+m\gamma_t^2}},\qquad
  E_t=A_{m,\ell}(\gamma_t)B_m(\gamma_t).
\end{equation}
Consequently, \cref{thm:main} applies with $q_t=R_tL_t$.
The exact sampling step $\CMstep{t}$ is $q_t$-Lipschitz in $W_2$, with
$q_t<1$ if and only if
\begin{equation}\label{eq:weak-curvature-contraction}
  \frac{\ell\gamma_t^2}{1+m\gamma_t^2}<\log(1+m\gamma_t^2).
\end{equation}
\end{theorem}
The proof is in \cref{sec:proof-gause-noise-CM}. When $\ell\le m$, \eqref{eq:weak-curvature-contraction} holds at every noise ratio, since $\log(1+z)>z/(1+z)$ for $z>0$; when $\ell>m$
it fails for small $\gamma_t$, as $q_t=1+(\ell-m)\gamma_t^2/2+O(\gamma_t^4)$,
but holds for large $\gamma_t$, since $q_t\to0$: larger noise restores certified contraction at the cost of a larger error budget in \eqref{eq:weak-curvature-profiles}.

A concrete sufficient condition for \eqref{eq:weakly-convex} is that
$f\in C^2$ has $\nabla^2 f\succeq m\eye$ outside a ball of radius
$\varrho>0$ and $\nabla^2 f\succeq-\beta\eye$ inside, with $\beta\ge0$.
One may then take
$\ell=2\max\{m+\beta,(m+\beta)^2\varrho^2\}$; the segment argument
justifying this choice is given in \cref{sec:app_proofs}.

\changenote{Paragraph on trained consistency models moved to \cref{sec:relaxed-noise-CM}.}

\begin{remark}[Comparison with \cite{Lyu2024SamplingEasyCM}]\label{rem:lyu}
Their analysis already establishes a multistep error recursion and geometric reduction toward an error plateau under a Lipschitz-type condition on the trained map. We instead derive the propagation factor from the target's curvature, following the direction proposed in their Section~4. Under strong log-concavity, \cref{coro:strong-log-concave} gives $q_t = (1+m\gamma_t^2)^{-\nicefrac{1}{2}} < 1$ at every $\gamma_t>0$, and \cref{thm:weak-curvature} identifies the levels that certify contraction under a curvature defect. Combined with $E_t$, these factors turn the balance condition of \cref{thm:main} into a quantitative criterion for reducing the noise while controlling accumulated error.

The two analyses rest on different assumptions: our exact-step recursion uses the uniform approximation bound \cref{ass:oneshot-approximation}, whereas their estimates use $L^2$ training losses and include numerical discretization. \Cref{thm:trained-cm} connects our flow estimates to consistency losses, with network Lipschitz constants still weighting the cellwise flow errors, and \cref{sec:relaxed-noise-CM} recovers their affine stability recursion as the learned-map alternative, with constants detailed in \cref{sec:app_compars}.
\end{remark}

\begin{remark}[Comparison with \cite{chen2025convergenceconsistencymodelmultistep}]\label{rem:chen}
Their Theorem~2 assumes bounded target samples and learned outputs,
and $L^2$ consistency along exact probability-flow trajectories, but
neither Lipschitz denoisers nor target curvature. It therefore covers
maps \replaced{lacking}{for which} the stability required by \cref{thm:main}\deleted{ is unavailable}.
Their Appendix~D \replaced{truncates suitable unbounded targets}{treats suitable unbounded targets by truncation}.

The difference is how errors propagate. Their Eq.~(6) accumulates
errors in a KL bound between noisy marginals and converts it to $W_2$
using boundedness, taking a fourth root. Our direct $W_2$ recursion
requires exact-denoiser stability and uniform approximation
(\cref{ass:denoise,ass:oneshot-approximation}), but no bounded support.
These controls retain the attenuation by later steps: for two levels,
the first-step error contributes $q_{t_2}E_{t_1}$ and is attenuated
when $q_{t_2}<1$.

The gain can be substantial when both sets of controls are available.
For the bounded VE example with $\pi=\mathrm{Unif}[-1,1]$ in
\cref{sec:app_compars}, the same learned maps, initialization and two
sampling levels give our certificate $0.1746$, versus $0.7659$ from
their bound, even after the refinements detailed there. Thus stability
can give a sharper guarantee; this comparison of upper bounds for the
same sampler does not order actual sampling errors or establish a
general dominance.
\end{remark}

\subsection{Multistep bounds \replaced{estimated from a trained
  model}{from error at the target}}\label{sec:relaxed-noise-CM}
\replaced{An error estimate for a trained model, such as \cref{thm:trained-cm}
below, typically controls}{The estimate \eqref{eq:trained-cm-target} controls} one learned sampling
step started from $\pi$. During iteration, however, the input law
generally differs from $\pi$. To handle these inputs, \cref{thm:main}
assumes that the learned and exact sampling steps produce distributions
at $W_2$-distance at most $E_t$, whatever their common input law
(\cref{ass:oneshot-approximation}). The estimate at $\pi$ does not
guarantee this. We instead control how the learned denoiser responds
to a change in its input law. Fix finite nonnegative profiles
\replaced{$E^\pi_t$ and $\widetilde L_t$ for target error and stability, and set
$\widetilde q_t:=R_t\widetilde L_t$.}{$E^\pi_t,\widetilde L_t,\beta_t$ for target error, multiplicative
stability and additive defect, respectively. Set
$\widetilde q_t:=R_t\widetilde L_t$ and $\bar E_t:=E^\pi_t+\beta_t$.}

\begin{assumption}[One-step accuracy at the target]\label{ass:approx-oneshot}
  For any $t>0$, \( W_2\big(\pi,\CMstepA{t}(\pi)\big)\;\le\;E^\pi_t\).
\end{assumption}

\begin{assumption}[\replaced{Stability}{Affine stability} of approximate denoising]\label{ass:approx-denoise}
  For all $t>0$ and $\rho_1,\rho_2\in\mathcal P_2(\Xspace)$,
  $W_2\big(\DEnoiseA{t}(\rho_1),\,\DEnoiseA{t}(\rho_2)\big)
    \;\le\;\widetilde L_t\,W_2(\rho_1,\rho_2)\deleted{+\beta_t}$.
\end{assumption}

\deleted{\Cref{thm:trained-cm-full} verifies these assumptions on the training grid using \eqref{eq:trained-cm-target}, $\widetilde L_{\tau_k}=\ell_k$ and $\beta_{\tau_k}=0$, without requiring uniformly bounded consistency residuals.} \replaced{Compared with \cref{thm:main}, the}{The} tradeoff is that $R_t\widetilde L_t$
replaces $R_tL_t$\deleted{ and $\beta_t$ adds to each step's error budget}.
Contraction of the exact step does not control either learned quantity.

\begin{proposition}[Learned-map recursion]\label{prop:learned-recursion}
  Let \cref{ass:noising,ass:approx-oneshot,ass:approx-denoise} hold. Then, for any
  $\rho\!\in\!\mathcal P_2(\Xspace)$,
  \begin{equation}\label{eq:learned-one-step}
    W_2\big(\CMstepA{t}(\rho),\pi\big)\;\le\;
    \widetilde q_t\,W_2(\rho,\pi)+\replaced{E^\pi_t}{\bar E_t}\,,\quad \textrm{for all}\quad t>0
  \end{equation}
Hence, the results of \cref{thm:main}, its balance condition and geometric bound~\eqref{eq:cm-geometric-error-bound}, and the results of \cref{sec:gen-noise-CM} on repeated levels and added steps hold verbatim with $(q_t, E_t)$ replaced by $(\widetilde q_t, \replaced{E^\pi_t}{\bar E_t})$.
\end{proposition}
The proof is in \cref{sec:proof-gause-noise-CM}.

\begin{deletedblock}
The pointwise bound in \citet[Assumption~3.5]{Lyu2024SamplingEasyCM}
implies \cref{ass:approx-denoise} by coupling and Minkowski's inequality.
\Cref{sec:app_compars} gives their OU factors and training-error budgets.

Outputs in a ball of radius $R$, as in
\citet[Theorem~2]{chen2025convergenceconsistencymodelmultistep}, give
the trivial choice $\widetilde L_t=0$, $\beta_t=2R$ in
\cref{ass:approx-denoise}; a sharper certificate requires finer
stability estimates.
\end{deletedblock}
\begin{addedblock}
\begin{remark}[Affine stability]\label{rem:affine-stability}
\Cref{ass:approx-denoise} may be relaxed to
$W_2(\DEnoiseA t(\rho_1),\DEnoiseA t(\rho_2))\le\widetilde L_tW_2(\rho_1,\rho_2)+\beta_t$,
with an additive defect $\beta_t\ge0$. This covers maps that are Lipschitz
only up to an offset: the pointwise bound of
\citet[Assumption~3.5]{Lyu2024SamplingEasyCM} gives it by coupling and
Minkowski's inequality (\cref{sec:app_compars}), and outputs in a ball of
radius $R$, as in \citet[Theorem~2]{chen2025convergenceconsistencymodelmultistep},
give $\widetilde L_t=0$ and $\beta_t=2R$. In this relaxation, $\widetilde L_t$
is only the multiplicative constant of the inequality, not a Lipschitz
constant of the learned map: bounded outputs, for instance, are not
$0$-Lipschitz. The proof of
\cref{prop:learned-recursion} then gives \eqref{eq:learned-one-step} with
$E^\pi_t+\beta_t$ in place of $E^\pi_t$, and the recursion, the balance
condition and the plateau follow as before. What is lost is the fixed point
of \cref{rem:fixed-point}: for $\beta_t>0$ the learned step need not be a
contraction, and its iterates need not converge.
\end{remark}
\end{addedblock}

When both sets of assumptions hold, the certificates combine level by
level (\cref{cor:mixed}). For $\widetilde q_t<1$, the repeated-level
bound plateaus at $\replaced{E^\pi_t}{\bar E_t}/(1-\widetilde q_t)$\replaced{, and $\CMstepA{t}$ has a
unique fixed point within this distance of $\pi$
(\cref{rem:fixed-point}).}{. If also $\beta_t=0$,
$\CMstepA{t}$ has a unique fixed point within this distance of $\pi$
(\cref{rem:fixed-point}); for $\beta_t>0$, \cref{ass:approx-denoise}
alone guarantees neither a fixed point nor convergence of the iterates.}

\begin{movedblock}
\paragraph{From the \replaced{approximate-score}{learned-score} flow to a trained consistency model.}
We transfer \replaced{the bounds of \cref{sec:gauss-noise-CM}}{these bounds} to a trained CM\replaced{, whose learned map
  $\PhiA{t\to0}$ need not be the flow of an ODE}{ $f_\theta$}. Let $\PhiS{t\to s}$ and
$\Phi_{t\to s}$ denote the \replaced{approximate-score}{learned-score} and exact
flow maps from $t$ to $s$, respectively\added{. Here the approximate-score ODE
  uses, in place of $s_\theta$, the score approximation of the solver used in
  training, such as a pretrained score model in consistency distillation, let
  $e_r$ in \eqref{eq:approx-assumption} denote its error}. On a grid
$0=\tau_0<\cdots<\tau_N=T$, consistency training compares
\replaced{$\PhiA{\tau_{i+1}\to0}(x)$}{$f_\theta(x,\tau_{i+1})$} with
\replaced{$\PhiA{\tau_i\to0}(\PhiS{\tau_{i+1}\to\tau_i}(x))$}{$f_\theta(\widetilde\Phi_{\tau_{i+1}\to\tau_i}(x),\tau_i)$}
for $x\sim\varpi_{\tau_{i+1}}$~\citep{song2023consistency,Lyu2024SamplingEasyCM}; here we
idealize the numerical solver by the
\replaced{approximate-score}{learned-score} flow.  \deleted{Set
  $\CMstep{\tau_k}^\theta(\rho):=(f_\theta(\cdot,\tau_k))_\#\OPnoise{\tau_k}(\rho)$.}  Define the cell-wise
error
\begin{equation}\label{eq:cell-error}
  E_{s,t}:=\alpha_s\int_{\gamma_s}^{\gamma_t}re_r
   \exp\Big(-\int_{\gamma_s}^{r}u\mu_u\D u\Big)\D r,\qquad 0\le s\le t.
\end{equation}
\replaced{For $s=0$ it is the profile \eqref{eq:general_Lt_Et} of this score
approximation}{This equals $E_t$ for $s=0$} and bounds
$\|\PhiS{t\to s}(x)-\Phi_{t\to s}(x)\|$ for every $x$
(\cref{lem:lip-error-ft})\added{: it controls the solver's flow, not the
  network}.
For each training cell, define the root-mean-square consistency error
\[
  \varepsilon_i:=\Big(\E_{X\sim\varpi_{\tau_{i+1}}}
    \big\|\replaced{\PhiA{\tau_{i+1}\to0}(X)}{f_\theta(X,\tau_{i+1})}
    -\replaced{\PhiA{\tau_i\to0}(\PhiS{\tau_{i+1}\to\tau_i}(X))}{f_\theta(\widetilde\Phi_{\tau_{i+1}\to\tau_i}(X),\tau_i)}\big\|^2\Big)^{\nicefrac{1}{2}},
  \qquad 0\le i<N.
\]

\begin{theorem}[Error of a trained CM]\label{thm:trained-cm}
Assume the hypotheses of \cref{theo:lip-error-ft}\added{ for the score
  approximation of the solver}, with
$(\Phi_{t\to s})_\#\varpi_t=\varpi_s$ for $0\le s\le t$.
Suppose \replaced{$\PhiA{0\to0}$}{$f_\theta(\cdot,0)$} is the identity\replaced{, so that $\widetilde L_{\tau_0}=1$,
and $\PhiA{\tau_i\to0}$ is $\widetilde L_{\tau_i}$-Lipschitz with
$\widetilde L_{\tau_i}<\infty$ for $1\le i\le N$}{ and
$f_\theta(\cdot,\tau_i)$ is $\ell_i$-Lipschitz with $\ell_i<\infty$
for $1\le i\le N$, and set $\ell_0:=1$}. Then, for $1\le k\le N$,
\begin{equation}\label{eq:trained-cm-target}
  W_2\big(\pi,\replaced{\CMstepA{\tau_k}}{\CMstep{\tau_k}^\theta}(\pi)\big)
  \le\sum_{i=0}^{k-1}\varepsilon_i
    +\sum_{i=0}^{k-1}\replaced{\widetilde L_{\tau_i}}{\ell_i}\,E_{\tau_i,\tau_{i+1}}.
\end{equation}
\end{theorem}
\added{Thus, on the grid, \cref{ass:approx-oneshot} holds with $E^\pi_{\tau_k}$
  the right-hand side of \eqref{eq:trained-cm-target}, and
  \cref{ass:approx-denoise} holds with the Lipschitz constants
  $\widetilde L_{\tau_k}$ of the network, without requiring uniformly bounded
  consistency residuals.}
\replaced{\Cref{ass:approx-oneshot} can also be obtained without the theorem:}{Alternatively,}
\cref{ass:oneshot,ass:oneshot-approximation} give $E^\pi_t:=E_t$\replaced{, since
then $W_2(\pi,\CMstepA{t}(\pi))=W_2(\CMstep{t}(\pi),\CMstepA{t}(\pi))$, and}{;}
the target error can \deleted{also }be estimated from samples
(\cref{sec:experiments}).
The full statement, including the uniform bound \replaced{with}{based on} sup-norm
consistency errors, and its proof are \deleted{given }in \cref{thm:trained-cm-full}.

\replaced{In both \eqref{eq:trained-cm-target} and its uniform version
\eqref{eq:trained-cm-uniform}}{In both bounds}, the factor $\replaced{\widetilde L_{\tau_i}}{\ell_i}$ replaces the contraction
of the exact flow built into $\replaced{E_{0,t}}{E_t}$: if $\replaced{\widetilde L_{\tau_i}}{\ell_i}=L_{\tau_i}$ the second sum
equals $\replaced{E_{0,\tau_k}}{E_{\tau_k}}$ exactly. Under the assumptions of
\cref{coro:strong-log-concave},
$E_{s,t}\le\alpha_s\sqrt{1+m\gamma_s^2}\,\big(B_m(\gamma_t)-B_m(\gamma_s)\big)$.
The uniform form \eqref{eq:trained-cm-uniform}, needed for the exact-map
route of \cref{thm:main}, requires finite sup-norm consistency errors.
The target bound \eqref{eq:trained-cm-target} instead \replaced{gives
\cref{ass:approx-oneshot} directly}{leads to multistep
bounds through learned-denoiser stability (\cref{sec:relaxed-noise-CM})}.
Using a numerical solver in place of $\PhiS{\tau_{i+1}\to\tau_i}$
adds $\replaced{\widetilde L_{\tau_i}}{\ell_i}$ times its local error to the \deleted{corresponding }consistency
budget: this error must be measured in $L^2(\varpi_{\tau_{i+1}})$ for
\eqref{eq:trained-cm-target} and in the sup norm for
\eqref{eq:trained-cm-uniform}.
\end{movedblock}

\begin{deletedblock}
\paragraph{An exact example of the effect of iteration.}
The preceding results compare error certificates. The same balance can
govern the actual error, as the following example shows. Take
$\pi=\mathcal N(0,1)$ with VE noising and define, for $t>0$,
$q_t=(1+t^2)^{-\nicefrac{1}{2}}$, $\Phi_{t\to0}(u)=q_tu$ and
$\widetilde\Phi_{t\to0}(u)=q_tu+b_t$, with $b_t\ge0$.
The exact map is the Gaussian probability-flow denoiser. The approximate
output is its translation by $b_t$, so its uniform one-step error is exactly
$E_t=b_t$. Starting from $U_0\sim\mathcal N(m_0,1)$ with $m_0\ge0$, the
sampler \eqref{eq:iterative-cm-sampler} satisfies
\begin{equation}\label{eq:gaussian-iteration-error}
  U_k\sim\mathcal N(m_k,1),\qquad
  m_k=q_{t_k}m_{k-1}+b_{t_k},\qquad
  W_2(\Law(U_k),\pi)=m_k.
\end{equation}
Indeed, noising and denoising preserve the unit variance and the
translation shifts the mean, so the recursion of \cref{thm:main} is
attained with equality at every step.

At fixed $t$, the error is
$q_t^nm_0+b_t(1-q_t^n)/(1-q_t)$. It decreases strictly if
$m_0>b_t/(1-q_t)$, increases strictly if $m_0<b_t/(1-q_t)$, and remains
constant at equality. The threshold $m_0$ versus $b_t/(1-q_t)$ is the
case distinction in the proof of
\citet[Corollary~3.11]{Lyu2024SamplingEasyCM}, here realized with
equality. Consequently, the same approximate denoiser can
benefit a poor initialization and worsen an accurate one. A decreasing
noise schedule can also first improve and then degrade the actual error:
\cref{tab:gaussian-iterations} evaluates \eqref{eq:gaussian-iteration-error}
with $b_t=0.05t$ and $t_k=2(0.85)^{k-1}$. Although the error added by each
step decreases, its contraction weakens; eventually the added error
outweighs the attenuation. Neither effect is in tension with
\citet{Lyu2024SamplingEasyCM}: their constant-level majorant (their
Eq.~(34)) is monotone toward the plateau while the actual error moves
toward it from either side, and along a decreasing schedule the example
shows that a certificate of the form of their Corollary~3.10 iterated can
be attained while the error first decreases and then increases.

\begin{table}[t]
  \centering
  \begin{tabular}{lrrrrr}
\toprule
\deleted{Initial mean / steps $n$} & \deleted{0} & \deleted{1} & \deleted{3} & \deleted{9} & \deleted{24} \\
\midrule
\deleted{$m_0=2$} & \deleted{2.0000} & \deleted{0.9944} & \deleted{0.4075} & \deleted{0.2298} & \deleted{0.2865} \\
\deleted{$m_0=0.05$} & \deleted{0.0500} & \deleted{0.1224} & \deleted{0.1559} & \deleted{0.1813} & \deleted{0.2527} \\
\bottomrule
\end{tabular}
  \caption{\deleted{Exact $W_2$ errors for the Gaussian sampler with
  $t_k=2(0.85)^{k-1}$ and $b_t=0.05t$, rounded to four decimals
  (analytic recursion, no sampling).}}
  \label{tab:gaussian-iterations}
\end{table}
\end{deletedblock}

\begin{addedblock}
\subsection{A Gaussian example, and tightness of certificates}\label{sec:gauss-example}
For a Gaussian target and an affine learned map, the CM iterates and their
$W_2$ errors are explicit. This case shows that the certificate is attained
for errors in the mean, overstates errors in the variance, and ignores
cancellations between variance errors of opposite signs. Derivations and
further regimes are given in \cref{sec:gauss-verification}. Let
$\pi=\mathcal N(0,m^{-1}\eye_d)$. The profiles of \cref{rem:VE-TrigFlow} are
then attained: $\Phi_{t\to0}(x)=L_tx$ and
$q_t=(1+m\gamma_t^2)^{-\nicefrac{1}{2}}$. Let the learned map be
$\PhiA{t\to0}(x)=\widetilde L_tx+\omega_tv$, where $\widetilde L_t$ is $L_t$
with a precision $\widetilde m>0$ in place of $m$, $v$ is a unit vector and
$\omega_t\ge0$. These parameters play the role of $\theta$: $\widetilde m$ sets the
variance the model believes in, and $\omega_t$ shifts its mean. The score it
implies, $\widetilde s_\theta(x,r)=-\widetilde\mu_rx+\eta_rv$ with
$\widetilde\mu_r=\widetilde m/(1+\widetilde mr^2)$ and an offset $\eta_r$
determined by $\omega_t$, is also affine (\cref{sec:gauss-verif-model}): the
example specifies the map, and the score is derived from it. Gaussian laws
remain Gaussian: from $\rho_0=\mathcal N(\zeta_0v,s_0^2\eye)$ and with
$\widetilde q_t=\alpha_t\widetilde L_t$, the iterates are
$\rho_k=\mathcal N(\zeta_kv,s_k^2\eye)$, where
\(\zeta_k=\widetilde q_{t_k}\zeta_{k-1}+\omega_{t_k}\),
\(s_k=\widetilde
L_{t_k}\big(\alpha_{t_k}^2s_{k-1}^2+\sigma_{t_k}^2\big)^{\nicefrac{1}{2}}\) and
$W_2(\rho_k,\pi)^2=\zeta_k^2+d\,(s_k-m^{-\nicefrac{1}{2}})^2$.

Errors in the mean attain the bounds exactly: if $\widetilde m=m$, the learned
map is the exact one plus the shift $\omega_tv$, so $E_t=\omega_t$ is the best
constant in \cref{ass:oneshot-approximation}; if also
$s_0=m^{-\nicefrac{1}{2}}$, the variance stays exact and
$W_2(\rho_k,\pi)=\zeta_k$ follows $\zeta_k=q_{t_k}\zeta_{k-1}+E_{t_k}$:
\eqref{eq:multi-stage-cm-err-bound} holds with equality for every
schedule. The trade-off of \cref{sec:gen-noise-CM} then governs the actual
error. Along a decreasing schedule it first falls and then rises, since
$q_{t_k}\to1$ faster than $E_{t_k}\to0$; at a repeated level it approaches
the plateau $E_t/(1-q_t)$ from either side, so one denoiser improves a poor
initialization and degrades a good one.

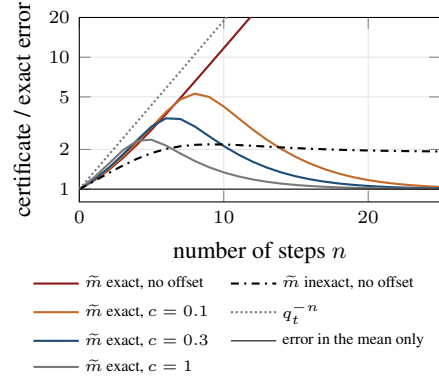
\begin{wrapfigure}{R}{0.45\linewidth}
  \vspace{-\baselineskip}
  \centering
  \definecolor{gaussblue}{HTML}{1F4E79}
\definecolor{gaussorange}{HTML}{B5651D}
\definecolor{gaussred}{HTML}{8B1A1A}
\definecolor{gaussgrey}{HTML}{7A7A7A}
\definecolor{gaussbrown}{HTML}{C26A2A}
\begin{tikzpicture}
\begin{axis}[scale only axis, clip mode=individual, grid=major, grid style={gray!20}, yticklabel style={/pgf/number format/1000 sep={}}, tick label style={font=\scriptsize}, label style={font=\small}, title style={font=\small}, legend style={font=\scriptsize, draw=none, fill=none}, legend cell align=left, width=0.76\linewidth, height=\mainfigheight, xmin=0, xmax=25,
  ymode=log, ymin=0.8, ymax=20, log ticks with fixed point, ytick={1,2,5,10,20},
  xlabel={number of steps $n$}, ylabel={certificate / exact error},
  legend columns=4, transpose legend,
  legend style={at={(0.42,-0.34)}, anchor=north, font=\tiny,
    /tikz/every even column/.append style={column sep=0.6em}}]
\addplot[gaussred, thick] table[x=n, y=r0] {Figures/gaussian/misses.dat};
\addlegendentry{$\widetilde m$ exact, no offset}
\addplot[gaussbrown, thick] table[x=n, y=r1] {Figures/gaussian/misses.dat};
\addlegendentry{$\widetilde m$ exact, $c=0.1$}
\addplot[gaussblue, thick] table[x=n, y=r2] {Figures/gaussian/misses.dat};
\addlegendentry{$\widetilde m$ exact, $c=0.3$}
\addplot[gaussgrey, thick] table[x=n, y=r3] {Figures/gaussian/misses.dat};
\addlegendentry{$\widetilde m$ exact, $c=1$}
\addplot[black, thick, dash dot] table[x=n, y=wrong] {Figures/gaussian/misses.dat};
\addlegendentry{$\widetilde m$ inexact, no offset}
\addplot[gaussgrey, densely dotted, thick] table[x=n, y=qref] {Figures/gaussian/misses.dat};
\addlegendentry{$q_t^{-n}$}
\addplot[black, thin] coordinates {(0,1) (25,1)};
\addlegendentry{error in the mean only}
\end{axis}
\end{tikzpicture}
  \caption{Certificate over exact error at the repeated VE
  level $t=1$, from $\mathcal N(0,2.5^2\eye)$, with $m=0.8$ and, when
  inexact, $\widetilde m=0.55$; offsets are $\omega_t=c\,B_{\widetilde m}(\gamma_t)$
  with $M_0=0.03$ and $M_1=0.05$.}
  \label{fig:gauss-verif-misses}
  \vspace{-\baselineskip}
\end{wrapfigure}
Errors in the variance are overstated. The constant $R_t=\alpha_t$ is attained by
two laws that differ by a shift, but adding independent noise also compresses
differences of standard deviation: for $\nu_s=\mathcal N(0,s^2\eye_d)$,
$W_2(\OPnoise t(\nu_s),\OPnoise t(\nu_{m^{-\nicefrac{1}{2}}}))
=\alpha_tW_2(\nu_s,\nu_{m^{-\nicefrac{1}{2}}})/\kappa_t(s)$ with $\kappa_t\ge1$, and $\kappa_t=q_t^{-1}$
at $s=m^{-\nicefrac{1}{2}}$. The certificate misses this factor, and the loss
compounds over steps: from a wrong variance and with an exact CM, it decays
like $q_t^n$ at a repeated level while the error eventually decays like
$q_t^{2n}$ (\cref{fig:gauss-verif-misses}). An error in the mean, being
certified exactly, hides this loss.

When a trained model does not reproduce the variance of $\pi$ exactly, i.e.,
$\widetilde m\ne m$, the learned and exact maps differ by $(\widetilde
L_t-L_t)x$, which is unbounded in $x$, so no uniform $E_t$ satisfies
\cref{ass:oneshot-approximation}. \Cref{prop:learned-recursion} still
applies, with $\widetilde L_t$ and the target error $E^\pi_t$ of
\eqref{eq:gauss-verif-target-error}, both attained. As before, the
certificate overstates the inherited variance error by $\kappa_t$, which an
error in the mean partly hides; at a repeated level its plateau
$E^\pi_t/(1-\widetilde q_t)$ exceeds the limiting error by a factor that
decreases to one as $\omega_t$ grows (\cref{fig:gauss-verif-misses} shows the case
without offset). Finally, the certificate cannot exploit cancellation: when the initialization
and the trained model err in the variance in opposite directions, their errors
partly cancel in the law but add up in the certificate
(cf. \cref{sec:gauss-verif-regimes}).
\par 
\end{addedblock}

\section{Numerical validation on Stable Diffusion (SD)~1.5}\label{sec:experiments}
We now show empirically that the contraction and approximation profiles entering our bounds can be measured reliably on trained CMs and closely match the functional forms predicted by the theory, indicating that our assumptions hold for real CMs and that we can quantify the constants involved.

We conduct this validation on a latent consistency model (LCM) distilled from Stable Diffusion (SD)~1.5~\citep{rombach2022ldm} by discrete-time consistency distillation, using a variance-exploding (VE) parametrization.\footnote{See \Cref{sec:estimators} for details \added{and implementation link}.} To cover the continuous-time setting, we also report results for an LCM trained on CelebA~\citep{liu2015deep} with the TrigFlow parametrization~\citep{lu2025simplifyingstabilizingscalingcontinuoustime} in the latent space of a KL-regularized VAE~\citep{rombach2022ldm}. For each model, we \replaced{approximate the target error $E^\pi_t$ and the contraction profile $\tilde q_t = R_t\tilde L_t$ of the learned map}{empirically estimate $E_t$ and the contraction profile $q_t = R_tL_t$} from \cref{ass:approx-oneshot,ass:noising,ass:approx-denoise} by using optimal transport $W_2$ estimators (see \cref{sec:estimators}), and compare them with the predictions of \cref{coro:strong-log-concave} (see \cref{rem:VE-TrigFlow}) for strongly log-concave targets. We focus on LCMs for three reasons: LCMs underlie state-of-the-art few-step image generators~\citep{luo2023lcm}; the reduced dimensionality of the latent space makes empirical estimation of $W_2$ distances and Lipschitz constants more reliable; and the KL-regularized VAE of latent models~\citep{rombach2022ldm} pushes the latent distribution $\pi$ towards a Gaussian, so the functional forms derived for the strongly log-concave setting of \cref{coro:strong-log-concave} may suffice.

\begin{figure}[h]
    \centering
    \begin{subfigure}[t]{0.48\linewidth}
        \centering
        \begin{tikzpicture}
\begin{axis}[
  theoryplot,
  height=\mainfigheight, width=0.8\linewidth,
  xmin=-0.15, xmax=15.15,
  ymin=0.05, ymax=1.045,
  xtick={0,2,4,6,8,10,12,14},
  ytick={0.2,0.4,0.6,0.8,1.0},
  xlabel={VE time $t=\sigma$},
  ylabel={$\tilde q_t$ vs $q_t$},
  legend style={at={(0.97,0.72)}, anchor=north east},
]
\addplot[gaussblue, thick, mark=*, mark size=1.1pt, mark repeat=3]
  table[x=t,y=empirical] {Figures/theory/data/ve_q.dat};
\addlegendentry{approximate $\tilde q_t$}

\addplot[gaussorange, thick, dashed]
  table[x=t,y=analytical] {Figures/theory/data/ve_q.dat};
\addlegendentry{$q_t$}

\node[anchor=north east, font=\scriptsize,
  fill=white, draw=gray!35, rounded corners=1pt, inner sep=2pt]
  at (rel axis cs:0.98,0.97) {$m=0.801$};
\end{axis}
\end{tikzpicture}%
        \caption{SD~1.5 LCM in VE coordinates, $m=0.801$.}
        \label{fig:q-ve}
    \end{subfigure}
        \hfill
        \begin{subfigure}[t]{0.48\linewidth}
        \centering
        \begin{tikzpicture}
\begin{axis}[
  theoryplot,
  height=\mainfigheight, width=0.8\linewidth,
  xmin=-0.03, xmax=1.58,
  ymin=0, ymax=1.045,
  xtick={0,0.2,0.4,0.6,0.8,1.0,1.2,1.4},
  ytick={0,0.2,0.4,0.6,0.8,1.0},
  xlabel={TrigFlow time $t$},
  ylabel={$\tilde q_t$ vs $q_t$},
  legend style={at={(0.02,0.03)}, anchor=south west},
]
\addplot[gaussblue, thick, mark=*, mark size=1.1pt, mark repeat=3]
  table[x=t,y=empirical] {Figures/theory/data/trigflow_q.dat};
\addlegendentry{approximate $\tilde q_t$}

\addplot[gaussorange, thick, dashed]
  table[x=t,y=analytical] {Figures/theory/data/trigflow_q.dat};
\addlegendentry{$q_t$}

\node[anchor=north east, font=\scriptsize,
  fill=white, draw=gray!35, rounded corners=1pt, inner sep=2pt]
  at (rel axis cs:0.98,0.97) {$m=0.5$};
\end{axis}
\end{tikzpicture}%
        \caption{TrigFlow LCM on CelebA, $m=0.5$.}
        \label{fig:q-trigflow}
    \end{subfigure}
    \caption{\replaced{Approximate $\tilde q_t$}{Empirical estimate $\widehat{\tilde{q}}_t$}
    and the predicted profile $q_t$ of \cref{coro:strong-log-concave}.  The analytical
    curve is used as a fitted functional form; it is not a certificate of the
    learned map's global contraction factor.}
    \label{fig:q-profiles}
\end{figure}
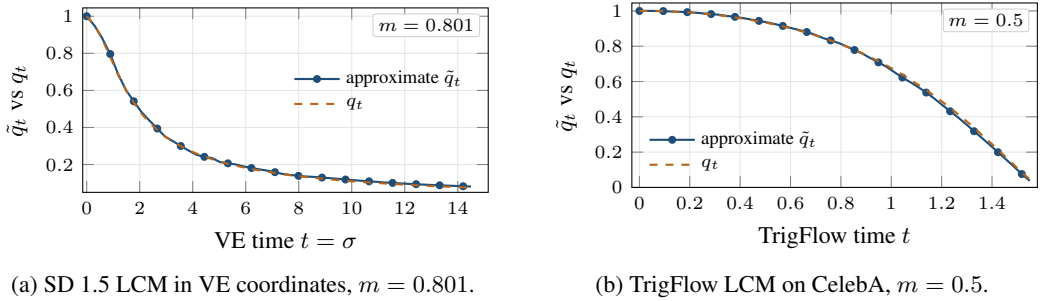

\begin{figure}[h]
    \centering
    \ref*{legend:E-target}\par
    \begin{subfigure}[t]{0.48\linewidth}
    \begin{tikzpicture}
\begin{axis}[
  theoryplot,
  height=\mainfigheight, width=0.77\linewidth,
  xmode=log,
  xmin=0.13, xmax=15.2,
  ymode=log,
  ymin=10, ymax=1000,
  xlabel={VE time $t=\sigma$},
  ylabel={$E_t^{\pi}$ vs $E_t$},
  legend columns=3,
  legend style={/tikz/every even column/.append style={column sep=1em}},
  legend to name=legend:E-target,
]
\addplot[draw=none, name path=veemp, forget plot]
  table[x=t,y=empirical] {Figures/theory/data/ve_E_bound.dat};
\addplot[draw=none, name path=vebound, forget plot]
  table[x=t,y=bound] {Figures/theory/data/ve_E_bound.dat};
\addplot[gaussorange!14, draw=none, forget plot]
  fill between[of=vebound and veemp];

\addplot[gaussblue, thick, mark=*, mark size=1.0pt, mark repeat=3]
  table[x=t,y=empirical] {Figures/theory/data/ve_E_bound.dat};
\addlegendentry{Approximate $E_t^{\pi}$}

\addplot[gaussorange, thick, dashed]
  table[x=t,y=bound] {Figures/theory/data/ve_E_bound.dat};
\addlegendentry{Analytical upper bound}

\addlegendimage{area legend, fill=gaussorange!14, draw=gaussorange!20}
\addlegendentry{Upper-bound gap}
\node[anchor=north west, font=\scriptsize,
  fill=white, draw=gray!35, rounded corners=1pt, inner sep=2pt, align=left]
  at (rel axis cs:0.02,0.96) {$m=0.801$\\$M_0=0$, $M_1=110.19$};
\end{axis}
\end{tikzpicture}%
    \caption{VE SD~1.5 LCM on AFHQ}
    \end{subfigure}
    \hfill
    \begin{subfigure}[t]{0.48\linewidth}
    \begin{tikzpicture}
\begin{axis}[
  theoryplot,
  height=\mainfigheight, width=0.77\linewidth,
  xmode=log,
  xmin=0.002, xmax=1.58,
  ymode=log,
  ymin=0.05, ymax=50,
  xlabel={TrigFlow time $t$},
  ylabel={$E_t^{\pi}$ vs $E_t$},
]
\addplot[draw=none, name path=tfemp, forget plot]
  table[x=t,y=empirical] {Figures/theory/data/trigflow_E_bound.dat};
\addplot[draw=none, name path=tfbound, forget plot]
  table[x=t,y=bound] {Figures/theory/data/trigflow_E_bound.dat};
\addplot[gaussorange!14, draw=none, forget plot]
  fill between[of=tfbound and tfemp];

\addplot[gaussblue, thick, mark=*, mark size=1.0pt, mark repeat=3]
  table[x=t,y=empirical] {Figures/theory/data/trigflow_E_bound.dat};

\addplot[gaussorange, thick, dashed]
  table[x=t,y=bound] {Figures/theory/data/trigflow_E_bound.dat};

\node[anchor=north west, font=\scriptsize,
  fill=white, draw=gray!35, rounded corners=1pt, inner sep=2pt, align=left]
  at (rel axis cs:0.02,0.96) {$m=0.5$\\$M_0=0$, $M_1=6.10$};
\end{axis}
\end{tikzpicture}%
    \caption{TrigFlow LCM on CelebA.}
    \end{subfigure}
    \caption{\replaced{Approximate target-only one-step error $E^\pi_t$}{Target-only one-step error diagnostic $\hat{\tilde E}_t^\pi$} for the
    TrigFlow and VE LCMs, compared with the analytical VE shape of
    \cref{coro:strong-log-concave} used as an upper envelope.}
    \label{fig:E-target-profile}
\end{figure}
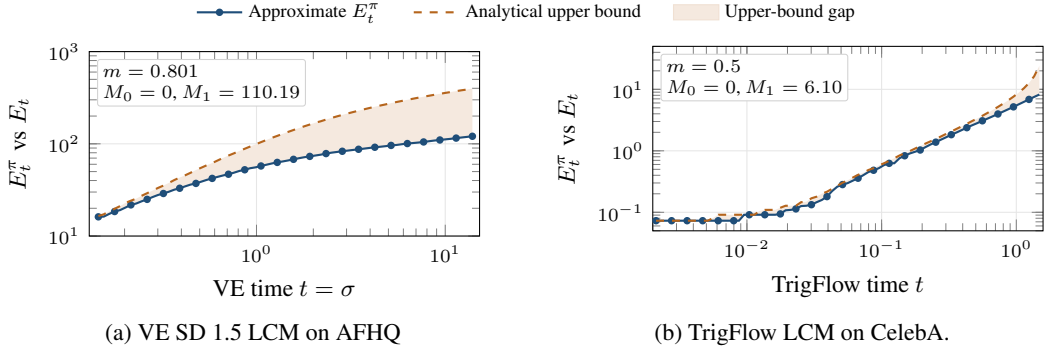

\Cref{fig:q-profiles} shows the \replaced{approximate}{estimated} contraction profile \replaced{$\tilde q_t$}{$\hat{\tilde{q}}_t = R_t \tilde{L}_t$} for both LCMs, alongside those predicted by \cref{coro:strong-log-concave} for a strongly log-concave target. They are contractive (\replaced{$\tilde q_t<1$}{$\hat{\tilde{q}}_t < 1$}) and closely follow the predicted forms, with fitted strong log-concavity constants $m \approx 0.8$ for the SD~1.5 LCM and $m \approx 0.5$ for the CelebA LCM. These agree with independent estimates from the covariance of the latent target, which give $m > 0.6$ and $m > 0.3$, respectively. This agreement supports the conclusion that our assumptions hold for these LCMs and that $\tilde{L}_t \approx L_t$. Moreover, \Cref{fig:E-target-profile} shows the \replaced{approximate}{estimated} error profile \replaced{$E^\pi_t$}{$\hat{E}^\pi_t$} for both LCMs, alongside the prediction $E_t$ of \cref{coro:strong-log-concave}. Since $E_t$ bounds the approximation error uniformly over inputs, whereas $E^\pi_t$ measures it only at the target, $E_t \geq E^\pi_t$ by definition, \replaced{so}{and} the predicted $E_t$ should \replaced{exceed the approximate $E^\pi_t$}{lie above $\hat{E}^\pi_t$}. Lastly, \Cref{app:one-step-bound} combines the estimated SD~1.5 profiles into the plug-in one-step $W_2$ bound and compares it with the empirical $W_2$ error on held-out data. The obtained bound lies above the error and tracks its variation with $t$.

\section{Conclusion}\label{sec:conclusion}
This paper derived a finite-horizon $W_2$ recursion for multistep CM sampling that separates contraction of the initialization error from accumulation of approximation error. Placing stability on the exact denoiser makes the contraction factor $q_t$ a property of the target, explicit under strong log-concavity and, above a noise threshold, under semi-log-concavity. Under a balance condition, the recursion yields a sharp geometric bound, and a complementary guarantee extends it to trained CMs under estimable assumptions. The experiments corroborate this analysis, as the estimated profiles closely follow the predicted functional forms. Future work includes extending the analysis to zero-shot samplers with inference-time guidance or rewards~\cite{Spagnoletti_2025_ICCV}, and optimizing schedules directly from the bound.

\section{Acknowledgments}

MP gratefully acknowledges support from UKRI Engineering and Physical Sciences Research Council (EPSRC) (EP/Z534481/1). ALHA acknowledges support from the the Alexander von Humboldt Foundation through an Alexander von Humboldt Research Fellowship for Experienced Researchers. AS and AA acknowledge support from the France 2030 research program on artificial intelligence via the PEPR PDE-AI grant (ANR-23-PEIA-0004). 
AD a reçu un financement de la Fondation de l'École polytechnique dans le cadre de sa campagne Servir la science.
AD is supported by the France 2030 program with the reference ANR-25-PEIA-0001 (THEOREM project).
AD is supported by Hi! Paris and Agence Nationale de la Recherche (Grant 11-LABX-0047).
This work received government funding administered by the National Research Agency (ANR) under the France 2030 program "Hi! PARIS", grant number ANR-23-IACL-0005.
EM was supported by the European Union (ERC-2022-SYGOCEAN-101071601).
The authors heavily rely on the \href{https://nova.mi.parisdescartes.fr/}{Nova HPC platform} of UFR Math-Info and MAP5 lab at Université Paris Cité. We are grateful to Azedine Mani and to Arnaud Meunier from the IT support unit for their technical help in maintaining the cluster nodes and assisting with software environment configuration. Additional HPC resources provided by \href{https://www.genci.fr/en/institut-du-developpement-et-des-ressources-en-informatique-scientifique-idris}{GENCI-IDRIS Jean-Zay} (Grants 
2024-AD011014557R3 and 2026-AD011017639).
We are grateful to Jean-Zay's IT support team for their assistance.

\bibliography{references}
\bibliographystyle{paperstyle}

\appendix
\crefalias{section}{appendix}
\section{Curvature under Gaussian smoothing and stability of the probability flow}\label{sec:app_proofs}
The \replaced{approximate-score}{learned-score} ODE introduced in \cref{sec:gauss-noise-CM} is
\begin{equation}\label{eq:grad-flow-approx}
  \dot x_t=a_tx_t-\tfrac12 b_t^2\replaced{s_\theta(x_t,t)}{s_t(x_t)}.
\end{equation}

\begin{lemma}\label{lem:lip-error-ft}
  Under the assumptions of \cref{theo:lip-error-ft}, let $x_u,y_u$ solve
  \eqref{eq:grad-flow} and \eqref{eq:grad-flow-approx}, respectively.
  For $0\lec s\lec t$,
  \[
    \norm{x_s-y_s}\le L_{s,t}\norm{x_t-y_t}+E_{s,t},
  \]
  where
  \[
    \begin{aligned}
      L_{s,t}&=\frac{\alpha_s}{\alpha_t}
        \exp\p*{-\int_{\gamma_s}^{\gamma_t}r\mu_r\D r},\\
      E_{s,t}&=\alpha_s\int_{\gamma_s}^{\gamma_t}r e_r
        \exp\p*{-\int_{\gamma_s}^{r}u\mu_u\D u}\D r.
    \end{aligned}
  \]
  Both constants are finite, with the improper-integral convention of
  \cref{theo:lip-error-ft} at $s=0$.
\end{lemma}
\begin{proof}
  Fix $0<s<t$. On this interval set $z_u=x_u/\alpha_u$,
  $w_u=y_u/\alpha_u$ and $c_u=\gamma_u\dot\gamma_u\ge0$.
  The identity
  $\nabla\log\varpi_u(x)=-\alpha_u^{-1}\nabla f_{\gamma_u}(x/\alpha_u)$
  and $b_u^2=2\alpha_u^2c_u$ give, almost everywhere,
  \[
    \dot z_u=c_u\nabla f_{\gamma_u}(z_u),\qquad
    \dot w_u=-c_u\replaced{\widetilde s_\theta(w_u,\gamma_u)}{\widetilde s_{\gamma_u}(w_u)}.
  \]
  By the curvature bound and \eqref{eq:approx-assumption}, the absolutely
  continuous function $h_u=\norm{z_u-w_u}$ satisfies
  \[
    \dot h_u\ge c_u\mu_{\gamma_u}h_u-c_u e_{\gamma_u}
    \quad\text{almost everywhere}.
  \]
  For $h_u>0$, this follows by taking the scalar product of
  $\dot z_u-\dot w_u$ with $(z_u-w_u)/h_u$; on $\{h_u=0\}$,
  $\dot h_u=0$ almost everywhere, so the inequality still holds.
  Gaussian smoothing makes $\nabla^2f_r(0)$ continuous for $r>0$;
  hence $\mu_r\le\lambda_{\min}(\nabla^2f_r(0))$ bounds $\mu_r^+$ on compact
  positive-noise intervals. Together with the integrability assumptions
  and substitution for an absolutely continuous nondecreasing $\gamma$,
  this ensures that
  $c_u\mu_{\gamma_u}$ and $c_u e_{\gamma_u}$ are integrable on $[s,t]$.
  A backward integrating factor therefore gives
  \[
    h_s\le e^{-\int_s^t c_u\mu_{\gamma_u}\D u}h_t
      +\int_s^t c_v e_{\gamma_v}
          e^{-\int_s^v c_u\mu_{\gamma_u}\D u}\D v.
  \]
  Substituting $r=\gamma_v$ and multiplying by $\alpha_s$ gives the stated
  constants. This substitution needs no inverse of $\gamma$ and allows
  constant positive-noise intervals.

  To pass to $s=0$, let $s\downarrow0$. For $0<a\lec r\lec\gamma_t$,
  \[
    \exp\p*{-\int_a^r u\mu_u\D u}
      \le C_t:=\exp\p*{\int_0^{\gamma_t}u\mu_u^-\D u}<\infty.
  \]
  Since $\int_0^{\gamma_t}r e_r\D r<\infty$, dominated convergence yields
  the stated $E_{0,t}$, even if $\int_0^r u\mu_u\D u=+\infty$.
  The same improper-integral limit gives $L_{0,t}$; continuity of the
  trajectories and $\alpha_0=1$ give the left-hand side.
  These bounds also prove finiteness. The case $s=t$ is immediate.
\end{proof}

\begin{lemma}[Strongly log-concave case, pathwise form]\label{lem:strong-log-concave-pathwise}
Under the parametrization and flow-existence assumptions of
\cref{theo:lip-error-ft}, let $f=-\log\pi$ be $C^2$ and
$m$-strongly convex, $m>0$. Let $x_u,y_u$ be the exact and approximate
trajectories of \cref{lem:lip-error-ft}. Assume that
  there exist \(M_0, M_1 \geq 0\) such that
  \[
    e_r=\sup_{x\in\Xspace}\norm{\replaced{\widetilde s_\theta(x,r)}{\widetilde s_r(x)}-\nabla\log\pi_r(x)}
    \le M_0+M_1/r
  \]
  for all $r>0$, with $e_r$ measurable. Then
  \[
    \norm{x_{0} - y_{0}} \le L_{t} \norm{x_{t} - y_{t}} + E_{t}
  \]
  where
  \[
    L_t
    = \alpha_{t}^{-1} \p*{1 + m \gamma_{t}^{2} }^{-\nicefrac{1}{2}}
  \]
  and
  \[
    \begin{aligned}
      E_{t}
      &= \frac{M_{0}}{m}
        \p*{\sqrt{1+m \gamma_{t}^{2}} - 1}
        +
        \frac{M_{1}}{\sqrt{m}}
        \, \arcsinh\p*{\sqrt{m} \gamma_{t} }\,.
    \end{aligned}
  \]
\end{lemma}
\begin{proof}
  By \cref{lem:gauss-smoothing}, \(f_{r}\) is \(\mu_{r}\)-strongly convex for
  \(\mu_{r} = \frac{m}{m r^{2} + 1}\), hence by \cref{lem:lip-error-ft}, we have
  the result once we substitute \(\mu_{r}\), \(\alpha_{0}=1, \sigma_{0}=0\) and integrate.
\end{proof}

\begin{lemma}[{Curvature under Gaussian smoothing, \citep[Theorem 3.7-(b)]{saumard2014log}}]\label{lem:gauss-smoothing}
Suppose that $\pi$ is $m$-strongly log-concave, for $m \geq 0$. Then for any $r >0$, $\pi_r$ is $\mu_r$-strongly log-concave with $\mu_r = m/(1+mr^2)$. Since $\pi_r=\kappa_r*\pi$ is positive and $C^\infty$, this means that $f_r=-\log\pi_r\in C^\infty(\Xspace)$ satisfies $\nabla^2 f_r\succeq\mu_r\eye$ on $\Xspace$.

\end{lemma}

We use $\psi_\ell$ and the relaxed curvature condition from
\eqref{eq:psi-def}--\eqref{eq:weakly-convex}, allowing $m=0$ here.
For $m,\ell\ge0$, a positive density $\pi$ with $f=-\log\pi\in C^1$
satisfying \eqref{eq:weakly-convex} is called
\emph{$(m,\ell)$-strongly log-concave outside a ball} below.
This terminology refers to the condition on pairs of points; a Hessian
bound outside a spatial ball is a sufficient condition, as shown below.
The following lemmas establish the curvature bounds used in
\cref{thm:weak-curvature}.

\begin{lemma}[Elementary bounds on \(\psi_\ell\)]\label{lem:psi-domination}
  Let \(\psi_\ell\) be as in \eqref{eq:psi-def}.
  \begin{enumerate}[label=(\alph*)]
  \item\label{psi:upper} For all \(\ell\ge0\) and \(u\ge0\),
    \(\psi_\ell(u)\le\min\{\ell u^2,\,2\ell^{\nicefrac{1}{2}}u\}\).
  \item\label{psi:dom} Let \(\beta',B'\ge0\) and let
    \(\Delta:[0,\infty)\to\rset\) satisfy
    \(\Delta(u)\le\min\{\beta'u^2,\,2B'u\}\) for all \(u\ge0\).  Then
    \(\Delta\lec\psi_\ell\) on \([0,\infty)\) for
    \(\ell=2\max\{\beta',\,B'^2\}\).
  \end{enumerate}
\end{lemma}
\begin{proof}
  \ref{psi:upper} is \(\tanh z\le\min\{z,1\}\) with \(z=u\ell^{\nicefrac{1}{2}}/2\) (and is
  trivial when \(\ell=0\)).

  \ref{psi:dom} The function \(\tanh\) is concave on \([0,\infty)\) with
  \(\tanh0=0\), so \(\tanh z\ge z\tanh1\) for \(z\in[0,1]\) and
  \(\tanh z\ge\tanh1\) for \(z\ge1\), i.e.\ \(\tanh z\ge\tanh(1)\min\{z,1\}\).
  With \(z=u\ell^{\nicefrac{1}{2}}/2\) this gives the companion lower bound to
  \ref{psi:upper},
  \begin{equation}\label{eq:psi-sandwich}
    \tanh(1)\,\min\{\ell u^2,\,2\ell^{\nicefrac{1}{2}}u\}\;\le\;\psi_\ell(u),
    \qquad u\ge0,
  \end{equation}
  and, \(\min\) being nondecreasing in each argument,
  \(\Delta(u)\le\min\{\beta'u^2,2B'u\}\le\psi_\ell(u)\) for all \(u\ge0\) as
  soon as \(\tanh(1)\ell\ge\beta'\) and \(\tanh(1)\ell^{\nicefrac{1}{2}}\ge B'\).  Both hold for
  \(\ell=2\max\{\beta',B'^2\}\), because \(2\tanh1>1\) and
  \(\sqrt2\,\tanh1>1\).
\end{proof}

\paragraph{A sufficient condition for the relaxed curvature bound.}
Suppose $f\in C^2$, with $\nabla^2 f\succeq m\eye$ outside a ball of
radius $\varrho>0$ and $\nabla^2 f\succeq-\beta\eye$ inside, where
$m>0$ and $\beta\ge0$. For $u=\norm{x-y}>0$, the part of the segment
$[y,x]$ inside the ball has length at most $\min(u,2\varrho)$.
Integrating the Hessian along this segment therefore gives
\[
  \psAl{\nabla f(x)-\nabla f(y)}{x-y}
  \ge m u^2-(m+\beta)\min\{u^2,2\varrho u\}.
\]
The case $x=y$ is immediate. Applying \cref{lem:psi-domination}\ref{psi:dom}
with $\beta'=m+\beta$ and $B'=(m+\beta)\varrho$ proves
\eqref{eq:weakly-convex} with
$\ell=2\max\{m+\beta,(m+\beta)^2\varrho^2\}$.

\begin{lemma}[Weak curvature under Gaussian smoothing]
  \label{lem:curvature_weakly_convex}
  Suppose \deleted{that }$\pi$ is $(m,\ell)$-strongly log-concave outside a ball \replaced{as
  in}{in the
  sense of}~\eqref{eq:weakly-convex}, for some $m,\ell \geq 0$, \replaced{let}{and write}
  $f=-\log\pi$\replaced{ and}{.
  Set} $c_r = 1+mr^2$ for $r \geq 0$. Then,
  \begin{enumerate}[label=(\alph*)]
  \item \label{lem:weak_conve_heat_semigorup_a} With the convention
    $\pi_0=\pi$, for any $r\geq 0$, $\pi_r$ is
    $(m_r,\ell_r)$-strongly log-concave outside a ball, with
    \begin{equation}
      \label{eq:evolved-convexity-parameters}
      m_r=\frac{m}{c_r}=\frac{m}{1+mr^2}\, ,
      \qquad
      \ell_r=\frac{\ell}{c_r^2}=\frac{\ell}{(1+mr^2)^2}\, ;
    \end{equation}
    that is, $f_r=-\log\pi_r$ is \replaced{$C^1$}{continuously differentiable} and satisfies
    \eqref{eq:weakly-convex} with $(m_r,\ell_r)$ in place of $(m,\ell)$.
  \item \label{lem:weak_conve_heat_semigorup_b}
    For any $r>0$, $f_r$ is $C^\infty$ and
    $\nabla^2 f_r(x) \succeq (m_r-\ell_r)\eye$ for any $x\in\rset^d$.
  \end{enumerate}
\end{lemma}

\begin{proof}
  \ref{lem:weak_conve_heat_semigorup_a} The claim is the hypothesis when
  $r=0$, so let $r>0$ and define $h(x)=f(x)-\frac{m}{2}\norm{x}^2$, which
  is continuously differentiable. For any $x,y\in\rset^d$,
  \eqref{eq:weakly-convex} yields
  \begin{equation}
    \label{eq:h-weakly-convex}
    \psAl{\nabla h(x)-\nabla h(y)}{x-y}
    =
    \psAl{\nabla f(x)-\nabla f(y)}{x-y}
    -m\norm{x-y}^2
    \geq
    -\psi_\ell(\norm{x-y})\, .
  \end{equation}

  \emph{Step 1: $e^{-h}$ grows at most exponentially.}  Fix a unit
  vector $\theta\in\rset^d$ and $t>0$. Applying \eqref{eq:h-weakly-convex}
  with $x=t\theta$ and $y=0$, dividing by $t$, and using
  $\psi_\ell(t)\leq 2\ell^{\nicefrac{1}{2}}t$
  (\cref{lem:psi-domination}\ref{psi:upper}), we obtain
  \[
    \frac{\D}{\D t}h(t\theta)
    =\psAl{\nabla h(t\theta)}{\theta}
    \geq
    \psAl{\nabla h(0)}{\theta}-\frac{\psi_\ell(t)}{t}
    \geq
    \psAl{\nabla h(0)}{\theta}-2\ell^{\nicefrac{1}{2}}\, .
  \]
  Integrating over $t\in[0,\norm{x}]$ along the ray $\theta = x/\norm{x}$
  gives, for all $x\in\rset^d$,
  \begin{equation}
    \label{eq:h-linear-lower-bound}
    h(x)\geq h(0)+\psAl{\nabla h(0)}{x}-2\ell^{\nicefrac{1}{2}}\norm{x}\, .
  \end{equation}
  Consequently $e^{-h(y)}\leq e^{-h(0)}e^{A\norm{y}}$ with
  $A=\norm{\nabla h(0)}+2\ell^{\nicefrac{1}{2}}$. Hence, for every $s>0$, the function
  \begin{equation}
    \label{eq:heat-potential-h}
    h_s \defeq -\log\p*{\kappa_s\ast e^{-h}}
  \end{equation}
  is well defined and finite: the Gaussian convolution of a positive
  function with at most exponential growth is finite, positive, and
  $C^\infty$, since differentiation under the integral sign is justified
  by the domination $\abs{\partial^k_x\kappa_s(x-y)}e^{A\norm{y}}\leq
  C_{k,s}(1+\norm{x-y})^{k}\kappa_{2s}(x-y)e^{A\norm{y}}$, which is
  integrable in $y$ locally uniformly in $x$. We also set $h_0=h$.

  \emph{Step 2: propagation of the weak-convexity bound.}  We claim that,
  for every $s\geq 0$ and $x,y\in\rset^d$,
  \begin{equation}
    \label{eq:heat-preservation-h}
    \psAl{\nabla h_s(x)-\nabla h_s(y)}{x-y}
    \geq
    -\psi_\ell(\norm{x-y})\, .
  \end{equation}
  For $s=0$ this is \eqref{eq:h-weakly-convex}, so let $s>0$. Assume first
  $\ell>0$. Let $(P_t)_{t\geq 0}$ be the heat semigroup of a standard
  $d$-dimensional Brownian motion, $P_t g=\kappa_{\sqrt{t}}\ast g$, so that
  $h_s=-\log P_{s^2}e^{-h}$. By \eqref{eq:h-weakly-convex}, $h$ belongs to
  the class $\mathcal{F}_\ell$ of \cite[Theorem 2.1]{conforti2024weak},
  namely the set of continuously differentiable $g:\rset^d\to\rset$ such
  that $\psAl{\nabla g(x)-\nabla g(y)}{x-y}\geq -\psi_\ell(\norm{x-y})$ for
  all $x,y\in\rset^d$. That theorem states that $\mathcal{F}_\ell$ is
  invariant under the flow $g\mapsto -\log P_{t}e^{-g}$ for every
  $t \in [0,\infty)$; applied with $t=s^2$, it gives $h_s\in\mathcal{F}_\ell$,
  which is \eqref{eq:heat-preservation-h}. Assume now $\ell=0$. Since
  $\psi_0=0\geq-\psi_L$\deleted{ for every $L>0$}, $h$ belongs to $\mathcal{F}_L$ for
  every $L>0$, and the previous case yields
  $\psAl{\nabla h_s(x)-\nabla h_s(y)}{x-y}\geq-\psi_L(\norm{x-y})\geq
  -L\norm{x-y}^2$ for every $L>0$,  the last step by
  \cref{lem:psi-domination}\ref{psi:upper}. Letting $L\downarrow 0$ gives \eqref{eq:heat-preservation-h}
  with $\psi_0=0$.

  \emph{Step 3: rescaling.}  Since $f(y)=h(y)+m\norm{y}^2/2$, setting
  $\tau_r\defeq r/c_r^{\nicefrac{1}{2}}$, so that $\tau_r^2=r^2/c_r$, and completing
  the square, we have for all $x,y\in\rset^d$
  \begin{equation}
    \label{eq:weak-curvature-square-completion}
    \frac{\norm{x-y}^2}{2r^2}+\frac{m}{2}\norm{y}^2
    =
    \frac{\norm{y-x/c_r}^2}{2\tau_r^2}
    +\frac{m}{2c_r}\norm{x}^2\, .
  \end{equation}
  Consequently, using $\pi=e^{-f}$ and
  $(2\pi\tau_r^2)^{d/2}/(2\pi r^2)^{d/2}=c_r^{-d/2}$,
  \[
    \begin{aligned}
      \pi_r(x)
      &=(2\pi r^2)^{-d/2}\int_{\rset^d}
        e^{-\norm{x-y}^2/(2r^2)-m\norm{y}^2/2-h(y)}\,\D y \\
      &=
        c_r^{-d/2}
        e^{-m\norm{x}^2/(2c_r)}
        \p*{\kappa_{\tau_r}\ast e^{-h}}(x/c_r)
      \\&=
      c_r^{-d/2}
      e^{-m\norm{x}^2/(2c_r)}
      e^{-h_{\tau_r}(x/c_r)}\, .
    \end{aligned}
  \]
  Taking logarithms,
  \begin{equation}
    \label{eq:fr-decomposition}
    f_r(x)
    =
    \frac{d}{2}\log c_r
    +\frac{m}{2c_r}\norm{x}^2
    +h_{\tau_r}(x/c_r)\, ,
  \end{equation}
  so that, by Step 1, $f_r$ is $C^\infty$ and
  \[
    \nabla f_r(x)
    =\frac{m}{c_r}\,x+\frac{1}{c_r}\,\nabla h_{\tau_r}(x/c_r)\, .
  \]
  Therefore, for any $x,y\in\rset^d$, applying
  \eqref{eq:heat-preservation-h} with $s=\tau_r$ at the points $x/c_r$ and
  $y/c_r$,
  \begin{align*}
    \psAl{\nabla f_r(x)-\nabla f_r(y)}{x-y}
    &=
    \frac{m}{c_r}\norm{x-y}^2
    +
    \psAl{
      \nabla h_{\tau_r}(x/c_r)-\nabla h_{\tau_r}(y/c_r)
    }{
      x/c_r-y/c_r
    }\\
    &\geq
    \frac{m}{c_r}\norm{x-y}^2
    -\psi_\ell(\norm{x-y}/c_r)
    =
    m_r\norm{x-y}^2-\psi_{\ell_r}(\norm{x-y})\, ,
  \end{align*}
  where the last equality follows from the identity
  $\psi_\ell(u/c)=\psi_{\ell/c^2}(u)$, valid for all $u\geq 0$, $c>0$ and
  $\ell\geq 0$, which is immediate from the definition of $\psi_\ell$
  since $(\ell/c^2)^{\nicefrac{1}{2}}=\ell^{\nicefrac{1}{2}}/c$. This proves
  \ref{lem:weak_conve_heat_semigorup_a}.

  \ref{lem:weak_conve_heat_semigorup_b} Let $r>0$. By
  \eqref{eq:fr-decomposition} and Step 1, $f_r$ is $C^\infty$. Fix
  $x\in\rset^d$, a unit vector $v\in\rset^d$ and $\varepsilon>0$. Applying
  \ref{lem:weak_conve_heat_semigorup_a} at the points $x+\varepsilon v$ and
  $x$, and dividing by $\varepsilon^2$, gives
  \[
    \frac{1}{\varepsilon}
    \psAl{\nabla f_r(x+\varepsilon v)-\nabla f_r(x)}{v}
    \geq
    m_r-\frac{\psi_{\ell_r}(\varepsilon)}{\varepsilon^2}\, .
  \]
  By \cref{lem:psi-domination}\ref{psi:upper},
  $\psi_{\ell_r}(\varepsilon)/\varepsilon^2\leq\ell_r$ for every
  $\varepsilon>0$ (including when $\ell_r=0$).
  Letting
  $\varepsilon\downarrow 0$, the left-hand side converges to
  $v^\top\nabla^2 f_r(x)v$ because $f_r$ is $C^2$, and we obtain
  $v^\top\nabla^2 f_r(x)v\geq m_r-\ell_r$. Since $v$ is an arbitrary unit
  vector, the claimed Hessian bound follows.
\end{proof}

\section{Proofs of the main results}\label{sec:proof-gause-noise-CM}
\begin{proof}[Proof of \cref{thm:main}]
For $\rho\in\mathcal P_2(\Xspace)$, the triangle
inequality, $\CMstep{t}(\pi)=\pi$ (\cref{ass:oneshot}) and
\cref{ass:oneshot-approximation,ass:denoise,ass:noising} give
$W_2(\CMstepA{t}(\rho),\pi)
 \le W_2(\CMstepA{t}(\rho),\CMstep{t}(\rho))
    +W_2(\CMstep{t}(\rho),\CMstep{t}(\pi))
 \le E_t+L_tR_t\,W_2(\rho,\pi)$.
With $d_k:=W_2(\rho_k,\pi)$ this yields $d_k\le q_{t_k}d_{k-1}+E_{t_k}$,
and unrolling proves \eqref{eq:multi-stage-cm-err-bound}. Under the
balance condition, $a_k:=Q_{k+1,n}E_{t_k}$ satisfies
$a_{k-1}\le\xi_na_k$ for $k=2,\ldots,n$ and $a_n=E_{t_n}$, so
$a_k\le\xi_n^{n-k}E_{t_n}$, and summing the geometric series gives
\eqref{eq:cm-geometric-error-bound}.
\end{proof}

\paragraph{Remarks on the proof of \cref{thm:main}.}
The one-step estimate $W_2(\CMstepA{t}(\rho),\pi)\le E_t+q_tW_2(\rho,\pi)$
is applied to $\rho=\rho_{k-1}$, a law produced by approximate steps; this
is where the uniformity in the input law of
\cref{ass:oneshot-approximation} is used. No division by a stability
factor occurs, so zero factors are allowed, and neither $q_{t_k}<1$ nor an
ordering of the levels is needed for \eqref{eq:multi-stage-cm-err-bound}.
The balance condition
\begin{equation}\label{eq:cm-error-balance}
  q_{t_k}E_{t_{k-1}}\le\xi_n E_{t_k},\qquad k=2,\ldots,n,
\end{equation}
enters only through
$a_{k-1}=Q_{k+1,n}q_{t_k}E_{t_{k-1}}\le\xi_nQ_{k+1,n}E_{t_k}=\xi_na_k$,
obtained by multiplying \eqref{eq:cm-error-balance} by \replaced{$Q_{k+1,n}\ge0$}{the nonnegative
product $Q_{k+1,n}$}; it covers zero error bounds, and if $E_{t_n}=0$ all
$a_k$ vanish, even when some stability factors are zero. For $n=1$ the
\deleted{backward }induction is empty and \eqref{eq:cm-geometric-error-bound} reads
$E_{t_1}\lec E_{t_1}$.

\begin{proof}[Proof of \cref{prop:learned-recursion}]
  Inserting $\CMstepA{t}(\pi)$ and using
  \cref{ass:approx-denoise,ass:approx-oneshot,ass:noising},
  $W_2(\CMstepA{t}(\rho),\pi)
   \le W_2\big(\CMstepA{t}(\rho),\CMstepA{t}(\pi)\big)
      +W_2\big(\CMstepA{t}(\pi),\pi\big)
   \le\widetilde L_t\,W_2\big(\OPnoise{t}(\rho),\OPnoise{t}(\pi)\big)\deleted{+\beta_t}+E^\pi_t
   \le\widetilde q_t\,W_2(\rho,\pi)+\replaced{E^\pi_t}{\bar E_t}$.
  All iterates have finite second moments (\cref{sec:gen-noise-CM}),
  so this bound applies at every step. The remaining claims follow
  as in the proof of \cref{thm:main}.
\end{proof}

\begin{lemma}[Fixed-point tracking and comparison of error bounds]\label{lem:fixed-point-bound}
  Let \cref{ass:noising,ass:approx-oneshot,ass:approx-denoise} hold.
  Fix $T>0$, an integer $n\ge1$ and levels $t_0,\ldots,t_n\in(0,T]$
  such that \deleted{$\beta_{t_k}=0$, }$\widetilde q_{t_k}:=R_{t_k}\widetilde L_{t_k}\in[0,1)$
  and $E^\pi_{t_k}<\infty$ for $k=0,\ldots,n$.
  Suppose finite constants $H_k\ge0$ satisfy
  \[
    W_2\big(\CMstepA{t_k}(\rho),\CMstepA{t_{k-1}}(\rho)\big)\le H_k,
    \qquad \rho\in\mathcal P_2(\Xspace),\quad k=1,\ldots,n.
  \]
  Set $F_k:=E^\pi_{t_k}/(1-\widetilde q_{t_k})$ and
  $\widetilde Q_{k,j}:=\prod_{i=k}^j\widetilde q_{t_i}$, with empty
  products equal to one. For $\rho_0\in\mathcal P_2(\Xspace)$, define
  $\rho_k:=\CMstepA{t_k}(\rho_{k-1})$ for $k=1,\ldots,n$. Then
  \begin{equation}\label{eq:fixed-point-bound}
    W_2(\rho_n,\pi)\le B_{\rm fp}
    :=\widetilde Q_{1,n}\big(W_2(\rho_0,\pi)+F_0\big)+F_n
      +\sum_{k=1}^n\frac{\widetilde Q_{k,n}}{1-\widetilde q_{t_k}}H_k .
  \end{equation}
  The recursive bound
  $B_{\rm rec}:=\widetilde Q_{1,n}W_2(\rho_0,\pi)
     +\sum_{k=1}^n\widetilde Q_{k+1,n}E^\pi_{t_k}$
  satisfies the identity
  \begin{equation}\label{eq:fixed-point-excess}
    B_{\rm fp}-B_{\rm rec}
    =\widetilde Q_{1,n}\Big(F_0+F_1+\frac{H_1}{1-\widetilde q_{t_1}}\Big)
     +\sum_{k=2}^n\widetilde Q_{k,n}
       \Big(F_k-F_{k-1}+\frac{H_k}{1-\widetilde q_{t_k}}\Big).
  \end{equation}
  In particular, if
  \[
    \widetilde q_{t_k}\ge\widetilde q_{t_{k-1}}
    \quad\text{and}\quad
    E^\pi_{t_{k-1}}\le E^\pi_{t_k}+H_k,
    \qquad k=2,\ldots,n,
  \]
  then $B_{\rm fp}\ge B_{\rm rec}+\widetilde Q_{1,n}(F_0+F_1)$.
  The error condition holds when
  $E^\pi_{t_k}=W_2(\pi,\CMstepA{t_k}(\pi))$ at every visited level;
  for arbitrary upper bounds it must be checked separately.
  At a repeated level $t_0=\cdots=t_n=t$, choosing $H_k=0$ gives
  $B_{\rm fp}-B_{\rm rec}=2\widetilde q_t^{\,n}E^\pi_t/(1-\widetilde q_t)$.
  When $n=1$, the sum over $k\ge2$ and the comparison conditions are empty.
\end{lemma}
\begin{proof}
  \replaced{By \cref{ass:noising,ass:approx-denoise}, each}{Since $\beta_{t_k}=0$ for $k=0,\ldots,n$,
  \cref{ass:noising,ass:approx-denoise} imply that each} $\CMstepA{t_k}$ is a
  contraction of the complete metric space
  $(\mathcal P_2(\Xspace),W_2)$. Let $\nu_k$ be its unique fixed point.
  Applying \eqref{eq:learned-one-step} to $\nu_k$ and rearranging gives
  \begin{equation}\label{eq:fixed-point-distance}
    W_2(\pi,\nu_k)\le F_k,\qquad k=0,\ldots,n.
  \end{equation}
  Inserting $\CMstepA{t_k}(\nu_{k-1})$ between the fixed points yields
  \[
    W_2(\nu_{k-1},\nu_k)
    \le H_k+\widetilde q_{t_k}W_2(\nu_{k-1},\nu_k),
  \]
  hence $W_2(\nu_{k-1},\nu_k)\le H_k/(1-\widetilde q_{t_k})$.
  Contraction at level $t_k$ and the triangle inequality therefore give
  \[
    W_2(\rho_k,\nu_k)
    \le\widetilde q_{t_k}W_2(\rho_{k-1},\nu_{k-1})
       +\frac{\widetilde q_{t_k}}{1-\widetilde q_{t_k}}H_k.
  \]
  \replaced{Unrolling it with}{Unroll this inequality and use}
  $W_2(\rho_0,\nu_0)\le W_2(\rho_0,\pi)+F_0$ and
  $W_2(\rho_n,\pi)\le W_2(\rho_n,\nu_n)+F_n$ \replaced{gives}{to obtain}
  \eqref{eq:fixed-point-bound}.

  Since $E^\pi_{t_k}=(1-\widetilde q_{t_k})F_k$, summation by parts gives
  \[
    \sum_{k=1}^n\widetilde Q_{k+1,n}E^\pi_{t_k}
    =F_n-\widetilde Q_{1,n}F_1
      -\sum_{k=2}^n\widetilde Q_{k,n}(F_k-F_{k-1}).
  \]
  Subtracting $B_{\rm rec}$ from $B_{\rm fp}$ proves
  \eqref{eq:fixed-point-excess}. Under the two comparison conditions,
  for each $k=2,\ldots,n$,
  \[
    F_k-F_{k-1}+\frac{H_k}{1-\widetilde q_{t_k}}
    =\frac{E^\pi_{t_k}+H_k}{1-\widetilde q_{t_k}}
      -\frac{E^\pi_{t_{k-1}}}{1-\widetilde q_{t_{k-1}}}
    \ge E^\pi_{t_{k-1}}
       \Big(\frac1{1-\widetilde q_{t_k}}
            -\frac1{1-\widetilde q_{t_{k-1}}}\Big)\ge0.
  \]
  All weights and $H_1$ are nonnegative, proving the comparison.
  For the exact defects $e_k:=W_2(\pi,\CMstepA{t_k}(\pi))$,
  the reverse triangle inequality gives
  $|e_k-e_{k-1}|
    \le W_2(\CMstepA{t_k}(\pi),\CMstepA{t_{k-1}}(\pi))\le H_k$,
  which verifies the error condition when $E^\pi_{t_k}=e_k$.
  At a repeated level, $H_k=0$ is admissible and $F_k$ is constant;
  \eqref{eq:fixed-point-excess} then reduces to the stated formula.
\end{proof}

A uniform H\"older bound on the sampling operators supplies
$H_k=H|t_k-t_{k-1}|^h$, but only the consecutive-level bounds above are
needed. A decreasing schedule with $\widetilde q_t$ nonincreasing in $t$
satisfies the monotonicity condition; the comparison still requires
compatible error bounds.

\begin{corollary}[Combining the two routes]\label{cor:mixed}
  Let \cref{ass:oneshot,ass:noising,ass:denoise,ass:oneshot-approximation,ass:approx-oneshot,ass:approx-denoise}
  hold. For any schedule $t_1,\ldots,t_n$ and any
  $\rho_0\in\mathcal P_2(\Xspace)$, the certificates $\bar d_0:=W_2(\rho_0,\pi)$
  and
  \[
    \bar d_k:=\min\big\{q_{t_k}\bar d_{k-1}+E_{t_k},\;
                   \widetilde q_{t_k}\bar d_{k-1}+\replaced{E^\pi_{t_k}}{\bar E_{t_k}}\big\},
    \qquad k=1,\ldots,n,
  \]
  satisfy $W_2(\rho_k,\pi)\le\bar d_k$.
\end{corollary}
\begin{proof}
  Both one-step estimates hold at every step (proofs of \cref{thm:main}
  and \cref{prop:learned-recursion}) and are nondecreasing in the previous
  certificate; induct on $k$.
\end{proof}

\begin{remark}[Fixed points of the learned step]\label{rem:fixed-point}
  Under \cref{ass:noising,ass:approx-oneshot,ass:approx-denoise}, if
  \deleted{$\beta_t=0$ and }$\widetilde q_t<1$, $\CMstepA{t}$ is a contraction of
  $(\mathcal P_2(\Xspace),W_2)$ with a unique fixed point
  $\widetilde\pi^{(t)}$, and \eqref{eq:learned-one-step} at
  $\rho=\widetilde\pi^{(t)}$ gives
  $W_2(\pi,\widetilde\pi^{(t)})\le E^\pi_t/(1-\widetilde q_t)$: the plateau
  of the repeated-level bound is the certified distance from $\pi$ to the
  law the repeated sampler converges to, and the Gaussian example of \replaced{\cref{sec:gauss-example}}{\cref{sec:relaxed-noise-CM}}
  attains it. The fixed-point argument of \cref{lem:fixed-point-bound}
  additionally requires bounds on the change between consecutive sampling operators,
  uniformly over input laws. Its bound is no smaller than the unrolled
  recursion under the lemma's monotonicity and error-compatibility
  conditions; there is no such ordering for arbitrary error bounds.
\end{remark}

\begin{proof}[Proof of \cref{thm:lipschitz-noiser}]
Let $(X,Y)$ be an optimal coupling of $(\rho,\nu)$ and
$Z\sim\mathcal N(0,\eye)$ be independent of it. Then $\alpha_tX+\sigma_tZ$
and $\alpha_tY+\sigma_tZ$ couple $\OPnoise{t}(\rho)$ and
$\OPnoise{t}(\nu)$, so
$W_2(\OPnoise{t}(\rho),\OPnoise{t}(\nu))
 \le\alpha_t\big(\E\|X-Y\|^2\big)^{\nicefrac{1}{2}}=\alpha_tW_2(\rho,\nu)$,
with equality for $\rho=\delta_x$ and $\nu=\delta_y$ with $x\ne y$, whose images are
Gaussian with covariance $\sigma_t^2\eye$ and means $\alpha_tx$, $\alpha_ty$.
\end{proof}

\paragraph{Proof sketch for \cref{theo:lip-error-ft}.}
Dividing trajectories by $\alpha_u$ removes the linear drift. For exact
and approximate trajectories $x_u,y_u$, the rescaled distance
$h_u:=\norm{x_u-y_u}/\alpha_u$ satisfies
$\dot h_u\ge\gamma_u\dot\gamma_u(\mu_{\gamma_u}h_u-e_{\gamma_u})$
almost everywhere at positive times. Thus curvature controls backward
stability, while score errors act as a forcing term. A backward
integrating factor and the substitution $r=\gamma_u$ yield
\eqref{eq:general_Lt_Et}. Comparing two exact trajectories removes the
forcing and gives $L_t$; using the same terminal state for exact and
approximate trajectories gives $E_t$. These estimates imply linear
growth, preserving $\mathcal P_2(\Xspace)$, and transfer to $W_2$ by
transporting input couplings. The endpoint limit is justified by the
integrability assumptions in \cref{lem:lip-error-ft}.

\begin{proof}[Proof of \cref{theo:lip-error-ft}]
  Apply \cref{lem:lip-error-ft} at $s=0$, first to two exact trajectories
  ($e_r=0$), then to an exact and an approximate trajectory with the same
  terminal state. Since $\alpha_0=1$ and $\gamma_0=0$, this gives
  \[
    \begin{aligned}
      \norm{\Phi_{t\to0}(x)-\Phi_{t\to0}(y)}&\le L_t\norm{x-y},\\
      \norm{\replaced{\PhiA{t\to0}(x)}{\widetilde\Phi_{t\to0}(x)}-\Phi_{t\to0}(x)}&\le E_t.
    \end{aligned}
  \]
  Both constants are finite by the lemma. These inequalities imply linear
  growth of both Borel maps, since
  $\norm{\replaced{\PhiA{t\to0}(x)}{\widetilde\Phi_{t\to0}(x)}}\le
  \norm{\Phi_{t\to0}(0)}+L_t\norm{x}+E_t$; their pushforwards therefore
  preserve $\mathcal P_2(\Xspace)$.
  Pushing an optimal coupling of $\rho_1,\rho_2$ through the exact map gives
  $W_2(\DEnoise{t}(\rho_1),\DEnoise{t}(\rho_2))\le L_tW_2(\rho_1,\rho_2)$.
  Finally, coupling the exact and approximate outputs through the same
  $X_t\sim\OPnoise{t}(\rho)$ gives
  $W_2(\CMstep{t}(\rho),\CMstepA{t}(\rho))\le E_t$, uniformly over
  $\rho\in\mathcal P_2(\Xspace)$.
\end{proof}

\begin{proof}[Proof of \cref{coro:strong-log-concave}]
By \cref{lem:gauss-smoothing}, \(f_{r}\) is \(C^\infty\) and \(\mu_{r}\)-strongly convex with
  \(\mu_{r} = m/(1 + m r^{2})\); substituting \(\mu_r\) and \(e_r \leq M_0 + M_1/r\) into
\eqref{eq:general_Lt_Et} and integrating gives the result.
\end{proof}

\begin{proof}[Proof of \cref{thm:weak-curvature}]
\Cref{lem:curvature_weakly_convex} gives
$\nabla^2 f_r\succeq\mu_r\eye$ with
$\mu_r=m/(1+mr^2)-\ell/(1+mr^2)^2$. This profile is bounded and
$re_r\le M_0r+M_1$, so the integrability requirements of
\cref{theo:lip-error-ft} hold. Direct integration gives
$\exp\p*{-\int_0^r u\mu_u\D u}
  =A_{m,\ell}(r)/\sqrt{1+mr^2}$.
Substitution into \eqref{eq:general_Lt_Et} yields $L_t$. Since
$A_{m,\ell}$ is nondecreasing, the error integral there is at most
$A_{m,\ell}(\gamma_t)B_m(\gamma_t)$.
\Cref{thm:lipschitz-noiser} gives $R_t$; rearranging $R_tL_t<1$ gives
\eqref{eq:weak-curvature-contraction}.
\end{proof}

\begin{theorem}[Error of a trained consistency model, full statement]\label{thm:trained-cm-full}
  \added{Let $\PhiS{t\to s}$ and $\Phi_{t\to s}$, $0\lec s\lec t$, be the
  backward flow maps from $t$ to $s$ of the approximate-score ODE
  \eqref{eq:grad-flow-approx}, with $s_\theta$ replaced by the score
  approximation used by the training solver, and of the probability-flow ODE
  \eqref{eq:grad-flow}.} Assume the hypotheses of \cref{theo:lip-error-ft}\added{ for this score
  approximation, with $e_r$ in \eqref{eq:approx-assumption} denoting its
  error,} and that the exact
  flow transports the noised marginals, $(\Phi_{t\to s})_\#\varpi_t=\varpi_s$
  for $0\lec s\lec t$. Fix a grid $0=\tau_0<\cdots<\tau_N=T$ and let
  $E_{s,t}$ be defined by \eqref{eq:cell-error}\added{ with this $e_r$}. Let \replaced{the learned maps $\PhiA{\tau_i\to0}:\Xspace\to\Xspace$, $0\lec i\lec N$,}{$f_\theta:\Xspace\times\{\tau_0,\ldots,\tau_N\}\to\Xspace$}
  be Borel, with \replaced{$\PhiA{0\to0}$}{$f_\theta(\cdot,0)$} the identity\replaced{, so that
  $\widetilde L_{\tau_0}=1$, and $\PhiA{\tau_i\to0}$ Lipschitz with
  finite constant $\widetilde L_{\tau_i}$ for $1\lec i\lec N$. For $0\lec i<N$,
  define}{ and $f_\theta(\cdot,\tau_i)$
  Lipschitz with finite constant $\ell_i$ for $1\lec i\lec N$.
  Set $\ell_0:=1$ and, for $0\lec i<N$, define}
  \begin{align*}
    \varepsilon_i&:=\Big(\E_{X\sim\varpi_{\tau_{i+1}}}
      \big\|\replaced{\PhiA{\tau_{i+1}\to0}(X)}{f_\theta(X,\tau_{i+1})}
      -\replaced{\PhiA{\tau_i\to0}(\PhiS{\tau_{i+1}\to\tau_i}(X))}{f_\theta(\widetilde\Phi_{\tau_{i+1}\to\tau_i}(X),\tau_i)}\big\|^2\Big)^{\nicefrac{1}{2}},\\
    \varepsilon_i^\infty&:=\sup_{x\in\Xspace}
      \big\|\replaced{\PhiA{\tau_{i+1}\to0}(x)}{f_\theta(x,\tau_{i+1})}
      -\replaced{\PhiA{\tau_i\to0}(\PhiS{\tau_{i+1}\to\tau_i}(x))}{f_\theta(\widetilde\Phi_{\tau_{i+1}\to\tau_i}(x),\tau_i)}\big\|.
  \end{align*}
  \replaced{The operators $\DEnoiseA{\tau_k}$ and $\CMstepA{\tau_k}$ of
  \cref{sec:gen-noise-CM}}{Let $D^\theta_{\tau_k}(\rho):=(f_\theta(\cdot,\tau_k))_\#\rho$ and
  $\CMstep{\tau_k}^\theta:=D^\theta_{\tau_k}\circ\OPnoise{\tau_k}$.
  These operators} preserve $\mathcal P_2(\Xspace)$, all $\varepsilon_i$
  are finite, and, for $k=1,\ldots,N$,
  \begin{equation}
    W_2\big(\pi,\replaced{\CMstepA{\tau_k}}{\CMstep{\tau_k}^\theta}(\pi)\big)
    \le\sum_{i=0}^{k-1}\varepsilon_i
      +\sum_{i=0}^{k-1}\replaced{\widetilde L_{\tau_i}}{\ell_i}\,E_{\tau_i,\tau_{i+1}},
    \label{eq:trained-cm-target-full}
  \end{equation}
  If additionally $\varepsilon_i^\infty<\infty$ for $0\lec i<k$, then
  \begin{equation}
    \sup_{\rho\in\mathcal P_2(\Xspace)}
    W_2\big(\CMstep{\tau_k}(\rho),\replaced{\CMstepA{\tau_k}}{\CMstep{\tau_k}^\theta}(\rho)\big)
    \le\sum_{i=0}^{k-1}\varepsilon_i^\infty
      +\sum_{i=0}^{k-1}\replaced{\widetilde L_{\tau_i}}{\ell_i}\,E_{\tau_i,\tau_{i+1}}.
    \label{eq:trained-cm-uniform}
  \end{equation}
  Thus, \replaced{for schedules in the grid}{with $\DEnoiseA{\tau_k}:=D^\theta_{\tau_k}$ and schedules taken
  in the grid}, \cref{ass:approx-oneshot} holds with the right-hand side of
  \eqref{eq:trained-cm-target-full}, and \cref{ass:approx-denoise} holds with
  \replaced{the constants $\widetilde L_{\tau_k}$}{$\widetilde L_{\tau_k}=\ell_k$ and $\beta_{\tau_k}=0$}.
  If $\varepsilon_i^\infty<\infty$ for every $0\lec i<N$, then
  \cref{ass:oneshot-approximation} also holds\replaced{,}{ on the grid} with the
  right-hand side of \eqref{eq:trained-cm-uniform}.
\end{theorem}

\paragraph{Proof sketch for \cref{thm:trained-cm-full}.}
The proof telescopes
$\replaced{\PhiA{\tau_k\to0}(x)}{f_\theta(x,\tau_k)}-\Phi_{\tau_k\to0}(x)$ along the exact trajectory
through the grid: each cell contributes its consistency error and the
discrepancy between the two flows over the cell, carried to time zero by
$\replaced{\PhiA{\tau_i\to0}}{f_\theta(\cdot,\tau_i)}$.

\begin{proof}[Proof of \cref{thm:trained-cm-full}]
  By \cref{lem:lip-error-ft}, each exact cell flow is Lipschitz and the
  \replaced{approximate-score}{learned-score} cell flow is at a uniformly bounded distance from it.
  Both have at most linear growth. Since the network is Lipschitz and
  $\varpi_{\tau_i}\in\mathcal P_2(\Xspace)$, all $\varepsilon_i$ are finite;
  the network pushforwards also preserve $\mathcal P_2(\Xspace)$.
  Fix $k$ and $x\in\Xspace$, and set $x_{\tau_i}:=\Phi_{\tau_k\to\tau_i}(x)$
  for $0\lec i\lec k$, so that $x_{\tau_k}=x$ and
  $x_{\tau_0}=\Phi_{\tau_k\to0}(x)$. Since $\replaced{\PhiA{0\to0}}{f_\theta(\cdot,0)}$ is the
  identity,
  \[
    \replaced{\PhiA{\tau_k\to0}(x)}{f_\theta(x,\tau_k)}-\Phi_{\tau_k\to0}(x)
    =\sum_{i=0}^{k-1}\big[\replaced{\PhiA{\tau_{i+1}\to0}(x_{\tau_{i+1}})}{f_\theta(x_{\tau_{i+1}},\tau_{i+1})}
                          -\replaced{\PhiA{\tau_i\to0}(x_{\tau_i})}{f_\theta(x_{\tau_i},\tau_i)}\big].
  \]
  Write $\hat x_{\tau_i}:=\PhiS{\tau_{i+1}\to\tau_i}(x_{\tau_{i+1}})$
  and split each summand as
  $[\replaced{\PhiA{\tau_{i+1}\to0}(x_{\tau_{i+1}})}{f_\theta(x_{\tau_{i+1}},\tau_{i+1})}-\replaced{\PhiA{\tau_i\to0}(\hat x_{\tau_i})}{f_\theta(\hat x_{\tau_i},\tau_i)}]
   +[\replaced{\PhiA{\tau_i\to0}(\hat x_{\tau_i})}{f_\theta(\hat x_{\tau_i},\tau_i)}-\replaced{\PhiA{\tau_i\to0}(x_{\tau_i})}{f_\theta(x_{\tau_i},\tau_i)}]$.
  The restriction of the exact trajectory $u\mapsto\Phi_{\tau_k\to u}(x)$
  to $[\tau_i,\tau_{i+1}]$ and the trajectory
  $u\mapsto\PhiS{\tau_{i+1}\to u}(x_{\tau_{i+1}})$ solve
  \eqref{eq:grad-flow} and \eqref{eq:grad-flow-approx} with the same
  terminal state $x_{\tau_{i+1}}$, so \cref{lem:lip-error-ft}, in its $s=0$
  form for $i=0$, gives
  $\|\hat x_{\tau_i}-x_{\tau_i}\|\le E_{\tau_i,\tau_{i+1}}$, and the
  Lipschitz property of $\replaced{\PhiA{\tau_i\to0}}{f_\theta(\cdot,\tau_i)}$, the identity for $i=0$,
  bounds the second bracket by $\replaced{\widetilde L_{\tau_i}}{\ell_i}E_{\tau_i,\tau_{i+1}}$.
  If $\varepsilon_i^\infty<\infty$ for all $i<k$, the first
  bracket is bounded by $\varepsilon_i^\infty$ pointwise, which proves
  $\sup_x\|\replaced{\PhiA{\tau_k\to0}(x)}{f_\theta(x,\tau_k)}-\Phi_{\tau_k\to0}(x)\|
   \le\sum_i\varepsilon_i^\infty+\sum_i\replaced{\widetilde L_{\tau_i}}{\ell_i}E_{\tau_i,\tau_{i+1}}$.
  For $\rho\in\mathcal P_2(\Xspace)$ and $Y\sim\OPnoise{\tau_k}(\rho)$, the
  pair $(\Phi_{\tau_k\to0}(Y),\replaced{\PhiA{\tau_k\to0}(Y)}{f_\theta(Y,\tau_k)})$ couples
  $\CMstep{\tau_k}(\rho)$ and $\replaced{\CMstepA{\tau_k}}{\CMstep{\tau_k}^\theta}(\rho)$, which gives
  \eqref{eq:trained-cm-uniform}.
  For \eqref{eq:trained-cm-target-full}, take $X\sim\varpi_{\tau_k}$, so that
  $x_{\tau_{i+1}}=\Phi_{\tau_k\to\tau_{i+1}}(X)\sim\varpi_{\tau_{i+1}}$ by
  the transport hypothesis and the first bracket of cell $i$ has
  $L^2$ norm $\varepsilon_i$. Minkowski's inequality in $L^2$, the pathwise
  bound on the second brackets, $(\Phi_{\tau_k\to0})_\#\varpi_{\tau_k}=\pi$
  and the coupling $(\Phi_{\tau_k\to0}(X),\replaced{\PhiA{\tau_k\to0}(X)}{f_\theta(X,\tau_k)})$ of $\pi$ and
  $\replaced{\CMstepA{\tau_k}}{\CMstep{\tau_k}^\theta}(\pi)$ give \eqref{eq:trained-cm-target-full}.
\end{proof}

\paragraph{Early stopping.}
For a grid $0<{\replaced{\delta}{\epsilon}}=\tau_0<\cdots<\tau_N=T$ with
\replaced{learned maps $\PhiA{\tau_i\to\delta}$ that return to time
$\delta$ and $\PhiA{\delta\to\delta}$}{$f_\theta(\cdot,{\epsilon})$} the identity, the target is $\varpi_{\replaced{\delta}{\epsilon}}$
and noising must start from time ${\replaced{\delta}{\epsilon}}$. Define, for $t\ge{\replaced{\delta}{\epsilon}}$,
\[
  N_t^{({\replaced{\delta}{\epsilon}})}(\rho):=
  \Law\left(\frac{\alpha_t}{\alpha_{\replaced{\delta}{\epsilon}}}X
       +\alpha_t\sqrt{\gamma_t^2-\gamma_{\replaced{\delta}{\epsilon}}^2}\,Z\right),
  \qquad X\sim\rho,\quad Z\sim\mathcal N(0,\eye),\quad X\perp Z.
\]
At each grid level, use
$\mathcal L_t^{({\replaced{\delta}{\epsilon}})}:=(\Phi_{t\to{\replaced{\delta}{\epsilon}}})_\#\circ N_t^{({\replaced{\delta}{\epsilon}})}$
and $\replaced{\widetilde{\mathcal L}}{\mathcal L}_t^{\theta,({\replaced{\delta}{\epsilon}})}:=\replaced{(\PhiA{t\to\delta})_\#}{D_t^\theta}\circ N_t^{({\replaced{\delta}{\epsilon}})}$.
Under the remaining hypotheses of \cref{thm:trained-cm-full},
\eqref{eq:trained-cm-target-full} and \eqref{eq:trained-cm-uniform} hold with
$(\pi,\CMstep{t},\replaced{\CMstepA{t}}{\CMstep{t}^\theta})$ replaced by
$(\varpi_{\replaced{\delta}{\epsilon}},\mathcal L_t^{({\replaced{\delta}{\epsilon}})},\replaced{\widetilde{\mathcal L}}{\mathcal L}_t^{\theta,({\replaced{\delta}{\epsilon}})})$,
using this grid to define the residuals, the same cell errors $E_{s,t}$
for $s\ge{\replaced{\delta}{\epsilon}}$, and $\replaced{\widetilde L_{\tau_0}}{\ell_0}=1$.
Indeed, $N_t^{({\replaced{\delta}{\epsilon}})}(\varpi_{\replaced{\delta}{\epsilon}})=\varpi_t$, and the proof\deleted{ above}
telescopes from $\tau_0={\replaced{\delta}{\epsilon}}$. The noising \deleted{stability }factor becomes
$R_t^{({\replaced{\delta}{\epsilon}})}=\alpha_t/\alpha_{\replaced{\delta}{\epsilon}}$.

\section{Comparison with other multistep bounds}\label{sec:app_compars}

\paragraph{The multistep bound of \cite{Lyu2024SamplingEasyCM}}
Their sampler is \eqref{eq:iterative-cm-sampler}
for the OU process restarted from $p_\delta$, with
$\alpha_t=e^{-(t-\delta)}$, $\sigma_t^2=1-e^{-2(t-\delta)}$ and a learned
map returning to \deleted{time }$\delta>0$. Thus $\pi=p_\delta$ and
$R_t=e^{-(t-\delta)}$. Their initial law $\mathcal N(0,\eye)$ is \replaced{OU-invariant}{invariant
under OU noising}, so their first step is
$\CMstepA{T}(\mathcal N(0,\eye))$.

Their Assumption~3.5 is an affine stability bound on the training grid:
for a universal constant $C_{\rm Lip}$, one may write
$\|f_\theta(x,t)-f_\theta(y,t)\|\le\Lambda\|x-y\|+\beta_t$, where
$\Lambda:=C_{\rm Lip}L_f$ and
$\beta_t:=C_{\rm Lip}t(L_f\varepsilon_{\rm sc}+\varepsilon_{\rm cm})$.
Coupling the inputs and applying Minkowski's inequality gives \replaced{the affine
stability of \cref{rem:affine-stability} with $\widetilde L_t=\Lambda$ and
this $\beta_t$}{\cref{ass:approx-denoise} with $\widetilde L_t=\Lambda$}.
Using the approximation bound stated in their Theorem~3.7, a coupling
through the exact denoiser gives \cref{ass:approx-oneshot} with
$E^\pi_t:=\mathrm{err}(t)$ at grid levels, where
\[
  \mathrm{err}(t):=C\big[t^2(\varepsilon_{\rm cm}+L_f\varepsilon_{\rm sc})
   +t(\varepsilon_{\rm cm}+L_f\varepsilon_{\rm sc}+L_fL_s^{\nicefrac{3}{2}}d^{\nicefrac{1}{2}}h)
   +t^{\nicefrac{1}{2}}L_fL_sd^{\nicefrac{1}{2}}h\big],
\]
where $C$ is universal, $\varepsilon_{\rm sc},\varepsilon_{\rm cm}$ are
their $L^2$ score and consistency errors, $h$ the training-grid step and
$L_s$ the Lipschitz constant of the noised scores.
\replaced{By \cref{rem:affine-stability}, \cref{prop:learned-recursion} then applies
with factor $\Lambda R_t$ and error budget $\bar E_t:=\mathrm{err}(t)+\beta_t$,}{Thus $\bar E_t=\mathrm{err}(t)+\beta_t$, and
\cref{prop:learned-recursion} applies directly with factor $\Lambda R_t$
and error budget $\bar E_t$,} recovering the recursion underlying their
Theorem~3.8 and Corollary~3.10.
For two levels $t_1\ge t_2>\delta$, this yields
\[
  W_2(\rho_2,p_\delta)\le
  \Lambda^2e^{-(t_1+t_2-2\delta)}W_2(\rho_0,p_\delta)
  +\Lambda e^{-(t_2-\delta)}\bar E_{t_1}+\bar E_{t_2}.
\]
The distance to $p_{\rm data}$ adds a term of order
$(\sqrt d\vee\mathfrak m)\sqrt\delta$, \replaced{with}{where} $\mathfrak m^2$ \deleted{is }its
second moment (their Eq.~(35)).

\Cref{thm:main} instead gives
$W_2(\rho_2,\pi)\le q_{t_2}q_{t_1}W_2(\rho_0,\pi)
+q_{t_2}E_{t_1}+E_{t_2}$, with factors certified from the exact denoiser.
For an $m$-strongly log-concave $p_\delta$, since
$\gamma_t^2=e^{2(t-\delta)}-1$, \cref{coro:strong-log-concave} gives
$q_t=(1+m(e^{2(t-\delta)}-1))^{-\nicefrac{1}{2}}<1$ for every $t>\delta$ and
$E_t\le B_m(\sqrt{e^{2(t-\delta)}-1})$.
The affine recursion has a coefficient below one when
$t>\delta+\log\Lambda$. These thresholds use different assumptions:
ours requires curvature control and a uniform approximation bound;
theirs requires the affine stability condition and target-marginal
error bounds. Moreover, $\bar E_t$ has polynomial growth in $t$, whereas the displayed bound
on $E_t$ grows like $e^t$ when $M_0>0$, so the comparison does not order
the error budgets. Their Theorem~C.1 \replaced{constructs}{provides a sufficient construction
of} the affine stability bound under bounded support and uniform training
errors; those extra hypotheses are not required by their main results.

\paragraph{The multistep bound of \cite{chen2025convergenceconsistencymodelmultistep}}
Use VE noising $X_r=X_0+rZ$, with $Z\sim\mathcal N(0,\eye)$ independent
of $X_0$, and
sampling levels $r_1>r_2>0$ on a uniform training grid of spacing
$\Delta r$. Write $c:=\epsilon_{\rm CM}/\Delta r$, where
$\epsilon_{\rm CM}$ is their $L^2$ consistency error, and
$D_0:=W_2(\delta_0,\pi)$. For target samples and learned outputs bounded
in norm by $R$, their Theorem~2 gives the two-step certificate
\begin{equation}\label{eq:chen-two-step-certificate}
  B_{\rm Chen}^{(2)}:=\min\left\{2R,\;
  2R\left(\frac{D_0^2}{4r_1^2}
       +\frac{c^2r_1^2}{4r_2^2}\right)^{\nicefrac{1}{4}}+cr_2\right\}.
\end{equation}
Here we strengthen their Eq.~(6) by retaining the actual second moment
$D_0^2$ in their proof, Eq.~(15), and by applying the support bound $2R$.
The initial noisy law is $\mathcal N(0,r_1^2\eye)$ in both analyses:
in our notation it is $\OPnoise{r_1}(\delta_0)$. Under
\cref{ass:oneshot,ass:noising,ass:denoise,ass:oneshot-approximation},
\cref{thm:main} instead gives
\begin{equation}\label{eq:ours-two-step-certificate}
  B_{\rm ours}^{(2)}:=q_{r_2}q_{r_1}D_0+q_{r_2}E_{r_1}+E_{r_2}.
\end{equation}
Their consistency error alone does not imply the uniform bounds $E_r$.
The following example verifies these additional bounds and shows a
quantitative gain from contraction.

Take $\pi=\mathrm{Unif}[-1,1]$. Let $F_r$ be the CDF of $U+rZ$ for
$U\sim\pi$ independent of $Z$, and let $\Phi_{\rm N}$ be the standard
normal CDF.
The one-dimensional probability flow preserves quantiles, so its
exact endpoint map is $g_r(x)=2F_r(x)-1$. Differentiating gives
$g'_r(x)=\Phi_{\rm N}((x+1)/r)-\Phi_{\rm N}((x-1)/r)$, maximized at
$x=0$. Thus, with $a:=\sqrt{2/\pi}$,
\[
  q_r=\operatorname{Lip}(g_r)=2\Phi_{\rm N}(1/r)-1
  \le a/r,\qquad D_0=1/\sqrt3.
\]
For $\kappa>0$, use the learned maps
$\widetilde g_r=(1-\kappa r)g_r$ at $0<r\lec T<1/\kappa$, with
$\widetilde g_0(x)=x$. Their positive-time outputs lie in
$[-1,1]$.\footnote{We use boundedness at positive sampling times,
as required by the proof of their Theorem~2. The printed output bound
also includes time zero, \deleted{which is }incompatible with the global identity
condition in their Assumption~1.}
For the exact probability-flow trajectory $(Y_r)$ started from $U$,
$g_r(Y_r)=U$. Each grid residual therefore has $L^2$ norm
$\kappa\Delta r/\sqrt3$, so
$c=\kappa/\sqrt3$. The pointwise approximation error is at most
$\kappa r$, while the learned map sends the noised target to
$\Law((1-\kappa r)U)$. Hence admissible
profiles for our two recursions are
\[
  E_r=\sqrt3\,cr,\qquad E_r^\pi=cr,\qquad \deleted{\beta_r=0,\qquad}
  \widetilde q_r=(1-\kappa r)q_r.
\]
This uses the abstract propositions; no full-support curvature
assumption is needed. At $\kappa=0.01$, $T=r_1=2$, $r_2=1$ and
$\Delta r=1$, \eqref{eq:ours-two-step-certificate} is $0.1746$, whereas
\eqref{eq:chen-two-step-certificate} is $0.7659$. The learned-map
certificate of \cref{prop:learned-recursion} is $0.1600$.

The gain also concerns the error rate. For an integer $j\ge3$, take
$c=j^{-4}$, a unit training grid up to $T=j^2$, and the two sampling
levels $r_1=j^2$, $r_2=j$. These choices satisfy $\kappa T<1$.
Using $q_r\le a/r$ in \eqref{eq:ours-two-step-certificate} gives
\[
  B_{\rm ours}^{(2)}\le Kc^{\nicefrac{3}{4}},\qquad
  K:=a^2/\sqrt3+\sqrt3(1+a)<3\sqrt2.
\]
In contrast, for any decreasing schedule with $N\ge2$ steps on this
same grid, the sum inside the fourth root in their Eq.~(6) contains
a term $c^2(r_{i-1}/r_i)^2/4\ge c^2/4$. Their certificate is therefore
at least $\sqrt2\,c^{\nicefrac{1}{2}}$, even with the second-moment refinement and
the support cap in \eqref{eq:chen-two-step-certificate}.
Since $K/j<\sqrt2$, our two-step certificate is strictly smaller than
every such multistep certificate. This example shows that additional
stability and uniform accuracy can yield a sharper guarantee.
These comparisons concern upper bounds, not an ordering of actual
sampling errors.

\section{Experimental estimators and implementation details}\label{sec:estimators}
This appendix details the diagnostics used in \cref{sec:experiments}.  They
are finite-sample estimators \replaced{of the quantities in}{for the case $\beta_t=0$ of}
\cref{prop:learned-recursion}\deleted{; they do not estimate an additive stability defect}. To distill SD 1.5, we run the \texttt{diffusers} consistency distillation script (\url{https://github.com/huggingface/diffusers/tree/main/examples/consistency_distillation}) for 1000 epochs on Conceptual~12M~\citep{changpinyo2021conceptual12mpushingwebscale}. We evaluate all quantities at 50 VE noise levels $\{\sigma_j\}_{j=1}^{50}$ from the schedule \citet{karras2022elucidating}.

\subsection{Latent coordinates and experimental setup}
For a clean latent $z\sim\pi$ and $\epsilon\sim\mathcal N(0,\eye)$, write
$r=\gamma_t=\sigma_t/\alpha_t$ and use the normalized noisy coordinate
\begin{equation}
    u=z+r\epsilon,
    \qquad
    F_t(u):=\replaced{\PhiA{t\to0}}{\widetilde\Phi_{t\to0}}(\alpha_tu).
    \label{eq:app-normalized-coordinate}
\end{equation}
Thus $F_t(u)$ is the clean-latent prediction of the learned denoiser.  In the
VE experiment $\alpha_t=1$ and $r=t=\sigma$; for TrigFlow
$\alpha_t=\cos t$, $\sigma_t=\sin t$, and $r=\tan t$.

For SD~1.5, AFHQ-cat images are resized to $512\times512$ and encoded with the
SD~1.5 VAE, producing scaled latents of size $4\times64\times64$ ($d=16384$
after flattening).  The current pool contains $2048$ target latents and uses
the fixed prompt \texttt{a high quality photo of a cat}, without
classifier-free guidance.  The VE grid contains 50 positive levels, from
approximately $0.13445$ to $14.61465$.  All reported distances are in
unsquared latent Euclidean units.

For the CelebA TrigFlow experiment, images are resized and center-cropped with
the transform used to train the local VAE, and the VAE posterior mean is used
directly as the latent representation.  We use at most $50{,}000$ CelebA
images and a fixed $90/10$ train/validation split, giving
$45{,}000$ calibration latents and $5{,}000$ held-out latents in the run used
for the plots.  The latent dimension is $d=32$.  The profiles shown in
\cref{fig:q-trigflow,fig:E-target-profile} are calibrated on the training
latent pool.

The CelebA LCM uses the trigonometric noising map
\begin{equation}
    x_t=\cos(t)x+\sigma_{\rm data}\sin(t)z,
    \qquad \sigma_{\rm data}=1,
    \label{eq:app-celeba-trig-noising}
\end{equation}
with $z\sim\mathcal N(0,I)$.  We evaluate $100$ TrigFlow times obtained by
first constructing a Karras grid with $\sigma\in[0.002,80]$ and $\rho=7$, and
then mapping it to TrigFlow time through
$t=\arctan(\sigma/\sigma_{\rm data})$.  Hence the grid spans approximately
$t\in[0.002,1.558]$.

\subsection{Pairwise stability estimator}
The population quantity entering \cref{prop:learned-recursion} is
$\widetilde q_t=R_t\widetilde L_t$, where $\widetilde L_t$ is a Lipschitz
constant of the learned denoising operator.  At each level, we draw noisy
particles $u_i=z_i+r\epsilon_i$ and compute $v_i=F_t(u_i)$.  For random index
pairs $(a,b)$ we form the secant ratios
\begin{equation}
    s_{ab}(t)=
    \frac{\|v_a-v_b\|_2}{\|u_a-u_b\|_2+10^{-12}}.
    \label{eq:app-secants}
\end{equation}
Taking $\widetilde L_t$ to be the optimal Lipschitz constant gives
$\operatorname{Lip}(F_t)=\alpha_t\widetilde L_t$;
for the Gaussian noising maps considered here $R_t=\alpha_t$, so the same
ratio directly estimates the learned-step factor $\widetilde q_t$ when the
normalized coordinate $u$ is used.

Let $Q_p(t)$ denote the empirical $p$-quantile of the ratios.  We then estimate
\begin{equation}
    \widehat q_t=2Q_{0.995}(t)-Q_{0.99}(t),
    \label{eq:app-qhat}
\end{equation}
with $512$ noisy particles and $8000$ sampled pairs per level for \deleted{the }SD 1.5\deleted{ experiment}, and with $2048$ noisy particles and $12000$ sampled pairs per level for CelebA. This high-quantile extrapolation is designed to be less sensitive than the sample
maximum while probing the upper tail of the \deleted{observed }secants.

\subsection{Target-only error estimators}
The quantity required by \cref{ass:approx-oneshot} is the target defect
\begin{equation}
    e_t^\pi:=W_2\!\left(\pi,\CMstepA{t}(\pi)\right).
    \label{eq:app-target-defect}
\end{equation}
We estimate this quantity with the same exact finite-sample OT estimator used
for $\widehat D_0$.  For repeat $r$, draw target particles
$X_{r,i}\sim\pi$ and noises $Z_{r,i}\sim\mathcal N(0,I)$, and evaluate each
level directly from the same clean targets,
\begin{equation}
    Y_{r,i}(t)=
    \replaced{\PhiA{t\to0}}{\widetilde\Phi_{t\to0}}(\alpha_tX_{r,i}+\sigma_tZ_{r,i}).
\end{equation}
For two equal-size clouds $A=\{a_i\}_{i=1}^N$ and
$B=\{b_i\}_{i=1}^N$, we use the exact empirical transport distance
\begin{equation}
    \widehat W_2(A,B)
    =\left[
       \frac1N\min_{\tau\in S_N}
       \sum_{i=1}^N\|a_i-b_{\tau(i)}\|_2^2
      \right]^{\nicefrac{1}{2}}.
    \label{eq:app-empirical-ot}
\end{equation}
The target-only error estimate for repeat $r$ is therefore
\begin{equation}
    \widehat E_{t,r}^\pi
    :=
    \widehat W_2\!\left(
        \{X_{r,i}\}_{i=1}^N,
        \{Y_{r,i}(t)\}_{i=1}^N
    \right),
\end{equation}
and the reported profile is the average across repeats,
\begin{equation}
    \widehat E_t^\pi
    :=
    \frac1R\sum_{r=1}^R \widehat E_{t,r}^\pi.
    \label{eq:app-Epi}
\end{equation}
The resulting $\widehat E_t^\pi$ is a
finite-sample plug-in estimate of the distributional target defect
$e_t^\pi$, rather than a Monte Carlo estimate of a fixed-coupling upper cost.

Within each repeat, the clean particles and noise realizations are held fixed
across noise levels, while evaluations at different levels are not chained.
For the CelebA TrigFlow experiment we use the same empirical-OT routine,
cloud size, and repeat count as for $\widehat D_0$; in the run used for the
plots this corresponds to $N=2048$ and $R=3$.  The SD~1.5 experiment likewise
uses the same empirical-OT estimator and OT budget as its corresponding
$\widehat D_0$ computation.

\subsection{Relation between $E_t$ and $E_t^\pi$}
The distinction used in \cref{sec:experiments} is important.  The error
$E_t$ in \cref{ass:oneshot-approximation} must satisfy
\begin{equation}
    W_2\!\left(\CMstep{t}(\rho),\CMstepA{t}(\rho)\right)\le E_t
    \qquad\text{for every }\rho\in\mathcal P_2(\Xspace),
\end{equation}
and the analytical expression $B_m(\gamma_t)$ in
\cref{coro:strong-log-concave,rem:VE-TrigFlow} is derived for this uniform
quantity from a pointwise score-approximation bound.  By contrast,
$E_t^\pi$ in \cref{ass:approx-oneshot} only controls the learned step at the
target.  Since the exact step preserves $\pi$, any admissible uniform $E_t$
also gives an admissible target-only bound $E_t^\pi\lec E_t$, but the converse
need not hold.

For the two parametrizations we compare the measurements with the common
analytical expression
\begin{equation}
    B_m(r)=
    \frac{M_0}{m}\bigl(\sqrt{1+mr^2}-1\bigr)
    +\frac{M_1}{\sqrt m}\operatorname{arcsinh}(\sqrt m\,r),
    \qquad r=\gamma_t.
    \label{eq:app-E-shape}
\end{equation}
For VE, $r=t=\sigma$.  We keep $m=0.801$ from the stability fit and choose
$M_0=0$, $M_1=110.19$, giving the upper envelope shown in
\cref{fig:E-target-profile}.  For the CelebA TrigFlow experiment,
$r=\tan t$; keeping the independently fitted stability value $m=0.5$ and
using $M_0=0$, $M_1=6.10$ gives the analytical overlay shown in the TrigFlow
panel.  These $M_0,M_1$ values are empirical envelope fits, whereas the
functional dependence on $\gamma_t$ is the analytical shape predicted by
\cref{rem:VE-TrigFlow}.

\subsection{One-step empirical OT comparison}
For the one-step experiment, the initialization is
$\rho_0=\mathcal N(0,I_d)$.  For each repeat and each VE level $t$, we sample
$G_i^0\sim\rho_0$ and generate
\begin{equation}
    G_i(t)=\replaced{\PhiA{t\to0}}{\widetilde\Phi_{t\to0}}(G_i^0+tZ_i).
    \label{eq:app-one-step-samples}
\end{equation}
Each time is evaluated independently from the same initialization law; there
is no preceding step at a larger noise level.  The current sweep uses $N=128$
particles and $3$ repeats.

The empirical output-to-target distance is computed with the same exact
linear-assignment estimator \eqref{eq:app-empirical-ot}.  The same estimator
also provides $\widehat D_0$ between a Gaussian reference cloud and a target
cloud.  We then evaluate the plug-in version of
\eqref{eq:learned-one-step},
\begin{equation}
    \widehat B_1(t)=\widehat q_t\widehat D_0+\widehat E_t^\pi.
    \label{eq:app-one-step-plugin}
\end{equation}
The target-error profile $\widehat E_t^\pi$ is calibrated separately using
the same empirical-OT construction in \eqref{eq:app-Epi}.  Consequently,
\cref{fig:one-step-bound} compares the empirical finite-cloud OT of the
one-step samples with a recursive estimate whose two distance terms,
$\widehat D_0$ and $\widehat E_t^\pi$, are computed with the same OT
estimator, while $\widehat q_t$ is calibrated independently from the secant
ratios.

\section{The Gaussian example: derivations and further regimes}
\label{sec:gauss-verification}

This appendix derives the statements of \cref{sec:gauss-example}, in its
notation, and treats further regimes.

\subsection{Model}
\label{sec:gauss-verif-model}

Let $\pi=\mathcal N(0,m^{-1}\eye_d)$. Then $\pi_r=\mathcal N(0,(m^{-1}+r^2)\eye)$
and $\nabla^2f_r=\mu_r\eye$ with $\mu_r=m/(1+mr^2)$, so
\cref{lem:gauss-smoothing} is attained, and the probability-flow ODE
\eqref{eq:grad-flow} integrates to $\Phi_{t\to0}(x)=L_tx$ with
\begin{equation}\label{eq:gauss-verif-profiles}
  R_t=\alpha_t,\qquad
  L_t=\frac{1}{\alpha_t\sqrt{1+m\gamma_t^2}},\qquad
  q_t=\frac{1}{\sqrt{1+m\gamma_t^2}} ,
\end{equation}
the profiles of \cref{rem:VE-TrigFlow}. \Cref{ass:oneshot,ass:noising,ass:denoise}
hold with these constants, which are exact operator norms ($R_t$ is attained
on Dirac masses).

The learned map may err in its slope, which encodes the variance it
ascribes to the target, and by an offset along a unit vector $v$, which
displaces the mean:
\begin{equation}\label{eq:gauss-verif-map}
  \PhiA{t\to0}(x)=\widetilde L_t\,x+\omega_t\,v,
  \qquad
  \widetilde L_t=\frac{1}{\alpha_t\sqrt{1+\widetilde m\gamma_t^2}},
\end{equation}
with a learned precision $\widetilde m>0$ and offsets $\omega_t\ge0$, which play
the role of $\theta$. Without offset, $\omega\equiv0$, this is the
exact map of $\mathcal N(0,\widetilde m^{-1}\eye)$, a model of the right mean
and the wrong variance.

Two affine objects appear: the learned map \eqref{eq:gauss-verif-map}, which
the example specifies, and the score \eqref{eq:gauss-verif-learned-score} it
implies; in each, the slope encodes the variance and the offset the mean. To
connect the map with \cref{theo:lip-error-ft}, take offsets of the form
$\omega_t=\int_0^{\gamma_t}r\,\eta_r(1+\widetilde mr^2)^{-\nicefrac{1}{2}}\D r$ with
$\eta_r\ge0$. In the rescaled coordinates of \cref{lem:lip-error-ft} the
approximate-score ODE is then linear, with integrating factor
$(1+\widetilde m\gamma_u^2)^{-\nicefrac{1}{2}}$, and
\eqref{eq:gauss-verif-map} is its flow for the implied affine score
\begin{equation}\label{eq:gauss-verif-learned-score}
  \widetilde s_\theta(x,r)=-\widetilde\mu_r\,x+\eta_r\,v,
  \qquad
  \widetilde\mu_r=\frac{\widetilde m}{1+\widetilde mr^2}.
\end{equation}
Each $\widetilde s_\theta(\cdot,r)$ is the score of a Gaussian, of variance
$\widetilde m^{-1}+r^2$ and mean $\eta_r(\widetilde m^{-1}+r^2)v$; with an
offset, the mean depends on $r$, which \cref{theo:lip-error-ft} allows since
it treats the levels separately. The score error \eqref{eq:approx-assumption}
is $e_r=\eta_r$ if $\widetilde m=m$, and $e_r=+\infty$ otherwise, the slope
mismatch $(\mu_r-\widetilde\mu_r)x$ being unbounded: the offset is the score
error only when the slope is right. For $\widetilde m=m$, $\omega_t$ is therefore
the error profile $E_t$ of \eqref{eq:general_Lt_Et}. The figures use
$\eta_r=M_0+M_1/r$, for which $\omega_t=B_{\widetilde m}(\gamma_t)$ with $B_m$ as
in \cref{rem:VE-TrigFlow}; for $\widetilde m=m$ this offset is the
score-error envelope of \cref{coro:strong-log-concave}, attained with
equality. The learned step is then the exact one followed by a shift, which
moves every law by the same vector; hence $E_t$, $E^\pi_t=E_t$ and
$\widetilde L_t=L_t$ are the best constants for both \cref{thm:main} and
\cref{prop:learned-recursion}, whose certificates coincide.

The mean/variance recursion of \cref{sec:gauss-example} follows, since noising
maps $\mathcal N(\zeta v,s^2\eye)$ to
$\mathcal N(\alpha_t\zeta v,(\alpha_t^2s^2+\sigma_t^2)\eye)$, the learned map
is affine, and two Gaussians with proportional covariances are compared
coordinatewise. For $\widetilde m=m$, $s=m^{-\nicefrac{1}{2}}$ is preserved.

\subsection{How the two errors propagate}
\label{sec:gauss-verif-propagation}

By the mean/variance recursion of \cref{sec:gauss-example}, an error in the mean
contracts at exactly the certified $\widetilde q_t$. An error in the variance
contracts more strongly: for $\nu_s=\mathcal N(0,s^2\eye_d)$, rationalizing
the difference of standard deviations after noising gives
\begin{equation}\label{eq:gauss-verif-variance-rate}
  \frac{W_2\big(\OPnoise t(\nu_s),\OPnoise t(\nu_{m^{-\nicefrac{1}{2}}})\big)}
       {W_2(\nu_s,\nu_{m^{-\nicefrac{1}{2}}})}
  =\frac{\alpha_t}{\kappa_t(s)},
  \qquad
  \kappa_t(s)=\frac{(\alpha_t^2s^2+\sigma_t^2)^{\nicefrac{1}{2}}+(\alpha_t^2/m+\sigma_t^2)^{\nicefrac{1}{2}}}
                   {\alpha_t\,(s+m^{-\nicefrac{1}{2}})} .
\end{equation}
Since $(\alpha_t^2x^2+\sigma_t^2)^{\nicefrac{1}{2}}\ge\alpha_tx$ and its derivative lies
in $[0,\alpha_t]$, $\kappa_t\ge1$ is nonincreasing; it equals $q_t^{-1}$ at
$s=m^{-\nicefrac{1}{2}}$, tends to $q_t^{-1}+\sqrt m\,\gamma_t$ as $s\to0$ and to $1$ as
$s\to\infty$. The constant $R_t=\alpha_t$ is exact on Dirac masses, which
differ by a shift, but misses the factor $\kappa_t$ on the laws the sampler
visits.

Let $S_t(x)=\widetilde L_t(\alpha_t^2x^2+\sigma_t^2)^{\nicefrac{1}{2}}$ be the standard
deviation after one learned step from standard deviation $x$. After a step
from $\rho=\mathcal N(\zeta v,s^2\eye)$, $\zeta\ge0$, the error in the
standard deviation is the sum of a part inherited from $\rho$,
$\Delta^{\rm init}_t=S_t(s)-S_t(m^{-\nicefrac{1}{2}})$, and a part created by the step,
$\Delta^{\rm step}_t=S_t(m^{-\nicefrac{1}{2}})-m^{-\nicefrac{1}{2}}$; each is positive if it makes
the law too wide and negative if it makes it too narrow. Since
$\CMstepA t(\pi)=\mathcal N\big(\omega_tv,S_t(m^{-\nicefrac{1}{2}})^2\eye\big)$, the
target error of \cref{ass:approx-oneshot} is attained at
\begin{equation}\label{eq:gauss-verif-target-error}
  E^\pi_t=\norm{\big(\omega_t,\;\sqrt d\,\Delta^{\rm step}_t\big)}
  =\Big(\omega_t^2+\frac dm\Big[\Big(\frac{1+m\gamma_t^2}{1+\widetilde m\gamma_t^2}\Big)^{\nicefrac{1}{2}}-1\Big]^2\Big)^{\nicefrac{1}{2}} .
\end{equation}
Its two entries are the errors in the mean and in the standard deviation of
one learned step from $\pi$; as $\gamma_t\to\infty$ the second tends to that of
the learned target $\mathcal N(0,\widetilde m^{-1}\eye)$. Since
$\abs{\Delta^{\rm init}_t}=\widetilde q_t\abs{s-m^{-\nicefrac{1}{2}}}/\kappa_t(s)$ by
\eqref{eq:gauss-verif-variance-rate}, the exact error and the certificate
after one step are
\begin{equation}\label{eq:gauss-verif-onestep}
  \begin{aligned}
    W_2\big(\CMstepA t(\rho),\pi\big)
    &=\norm{\big(\widetilde q_t\zeta+\omega_t,\;
       \sqrt d\,(\Delta^{\rm init}_t+\Delta^{\rm step}_t)\big)},\\
    \widetilde q_t\,W_2(\rho,\pi)+E^\pi_t
    &=\norm{\big(\widetilde q_t\zeta,\;
       \sqrt d\,\kappa_t(s)\abs{\Delta^{\rm init}_t}\big)}
     +\norm{\big(\omega_t,\;\sqrt d\,\abs{\Delta^{\rm step}_t}\big)} .
  \end{aligned}
\end{equation}
The mean entries agree, so an error in the mean is certified exactly,
while the inherited variance error is overstated by $\kappa_t$, a bounded
factor per step that compounds over steps. Adding two norms instead of
combining their entries costs at most a further factor $\sqrt2$, and adding
the sizes of $\Delta^{\rm init}_t$ and $\Delta^{\rm step}_t$ instead of their
signed values ignores cancellation (\cref{sec:gauss-verif-regimes}).

At a repeated level the iterates converge to
$\rho_\star=\mathcal N(\zeta_\star v,s_\star^2\eye)$, with
$\zeta_\star=\omega_t/(1-\widetilde q_t)$ and
$s_\star=\widetilde L_t\sigma_t(1-\widetilde q_t^{\,2})^{-\nicefrac{1}{2}}$. Writing
$\Delta_t=\sqrt d\,\abs{\Delta^{\rm step}_t}$ and
$\delta_t=\sqrt d\,\abs{s_\star-m^{-\nicefrac{1}{2}}}$, the plateau of
\cref{rem:fixed-point} exceeds the exact limit by
\begin{equation}\label{eq:gauss-verif-plateau-ratio}
  \frac{E^\pi_t/(1-\widetilde q_t)}{W_2(\rho_\star,\pi)}
  =\sqrt{\frac{\omega_t^2+\Delta_t^2}{\omega_t^2+(1-\widetilde q_t)^2\delta_t^2}} ,
\end{equation}
which is at least one, by \eqref{eq:learned-one-step} at $\rho_\star$, and
decreases in $\omega_t$.

\subsection{Regimes}
\label{sec:gauss-verif-regimes}

\Cref{tab:gauss-verif-regimes} lists what the certificate overstates per
regime; the extremes isolate a single error source.

\begin{table}[h]
  \centering
  \small
  \begin{tabular}{@{}lllll@{}}
    \toprule
    Regime & Initialization & Learned step & Certificate overstates & Outcome \\
    \midrule
    Mean & mean & mean & nothing & exact \\
    Initial variance & variance & none & inherited variance, by $\kappa_t$
      & ratio $\sim q_t^{-n}$ \\
    Learned variance & variance or none & variance
      & inherited variance, by $\kappa_t$ & no uniform $E_t$ \\
    Both & both kinds & both kinds & the same, diluted by the mean
      & loose, then tight \\
    \bottomrule
  \end{tabular}
  \caption{Regimes of \cref{sec:gauss-verif-regimes}. The second and third
  columns give the kind of error introduced by the initialization and by the
  learned step, \emph{none} meaning exact. The certificate also adds the
  sizes of variance errors of opposite directions, which may cancel
  (\cref{sec:gauss-verif-propagation}).}
  \label{tab:gauss-verif-regimes}
\end{table}

\paragraph{Error in the mean.}
With $\widetilde m=m$ and $s_0=m^{-\nicefrac{1}{2}}$ the two lines of
\eqref{eq:gauss-verif-onestep} agree, and so do the unrolled
\eqref{eq:learned-one-step} and \eqref{eq:multi-stage-cm-err-bound}.
\Cref{fig:gauss-verif-translation} shows the statements of
\cref{sec:gauss-example} in both parametrizations; the plateau of
\cref{rem:fixed-point} is attained. The geometric form \eqref{eq:cm-geometric-error-bound} also needs the
balance condition \eqref{eq:cm-error-balance}, which for
$t_k=t_1\lambda^{k-1}$ holds only for slowly decaying schedules, a range
that narrows as $n$ grows.

\begin{figure}[t]
  \centering
  \definecolor{gaussblue}{HTML}{1F4E79}
\definecolor{gaussorange}{HTML}{B5651D}
\definecolor{gaussred}{HTML}{8B1A1A}
\definecolor{gaussgrey}{HTML}{7A7A7A}
\definecolor{gaussbrown}{HTML}{C26A2A}
\begin{tikzpicture}
\begin{groupplot}[scale only axis, clip mode=individual, grid=major, grid style={gray!20}, yticklabel style={/pgf/number format/1000 sep={}}, tick label style={font=\scriptsize}, label style={font=\small}, title style={font=\small}, legend style={font=\scriptsize, draw=none, fill=none}, legend cell align=left, group style={group size=2 by 1, horizontal sep=1.5cm},
  width=0.37\linewidth, height=3.6cm, xmin=0, xmax=24,
  xlabel={number of steps $n$}, ylabel={$W_2(\rho_n,\pi)$}]
\nextgroupplot[title={Certificate and error coincide}, ymode=log,
  log ticks with fixed point, ytick={0.3,0.5,1,2}, legend columns=2,
  legend pos=north east]
\addplot[gaussblue, thick] table[x=n, y=ve] {Figures/gaussian/translation-left.dat};
\addlegendentry{VE: error}
\addplot[gaussblue, only marks, mark=o, each nth point=2] table[x=n, y=ve] {Figures/gaussian/translation-left.dat};
\addlegendentry{certificate}
\addplot[gaussorange, thick] table[x=n, y=tf] {Figures/gaussian/translation-left.dat};
\addlegendentry{TrigFlow: error}
\addplot[gaussorange, only marks, mark=square, each nth point=2] table[x=n, y=tf] {Figures/gaussian/translation-left.dat};
\addlegendentry{certificate}
\addplot[gaussblue, only marks, mark=star, mark size=3.5pt, very thick, forget plot]
  coordinates {(9,0.313797)};
\addplot[gaussorange, only marks, mark=star, mark size=3.5pt, very thick, forget plot]
  coordinates {(7,0.32103)};
\nextgroupplot[title={Repeated level}, ymin=0, legend pos=north east]
\addplot[gaussblue, thick] table[x=n, y=above] {Figures/gaussian/translation-right.dat};
\addlegendentry{$\zeta_0$ above the plateau}
\addplot[gaussgrey, thick] table[x=n, y=at] {Figures/gaussian/translation-right.dat};
\addlegendentry{$\zeta_0$ at the plateau}
\addplot[gaussorange, thick] table[x=n, y=below] {Figures/gaussian/translation-right.dat};
\addlegendentry{$\zeta_0$ below the plateau}
\addplot[gaussred, dashed] table[x=n, y=plateau] {Figures/gaussian/translation-right.dat};
\addlegendentry{plateau $E_t/(1-q_t)$}
\end{groupplot}
\end{tikzpicture}
  \caption{Error in the mean only: $m=0.8$, the curvature fitted in
  \cref{sec:experiments}, $\widetilde m=m$, $\eta_r=M_0+M_1/r$ with
  $M_0=0.03$, $M_1=0.05$, and $s_0=m^{-\nicefrac{1}{2}}$. Left: from $\zeta_0=2$, along
  $t_k=2(0.85)^{k-1}$ in VE coordinates and $t_k=1.3(0.85)^{k-1}$ in TrigFlow
  coordinates, the certificate \eqref{eq:multi-stage-cm-err-bound} (markers)
  and the exact error (lines) coincide; stars mark the optimal step count.
  Right: at the repeated VE level $t=1$, the error moves to the plateau
  $E_t/(1-q_t)$ from above, at, or below it.}
  \label{fig:gauss-verif-translation}
\end{figure}
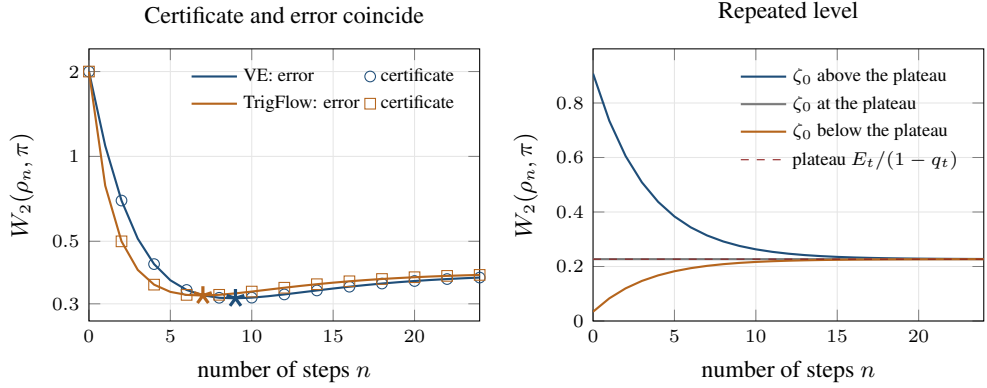

\paragraph{Variance error from the initialization.}
With an exact CM ($\widetilde m=m$, $\eta\equiv0$, hence $E_t=E^\pi_t=0$)
started from $\mathcal N(0,s_0^2\eye)$, only the factor $\kappa_t$ acts: the
certificate is the pure contraction $Q_{1,n}W_2(\rho_0,\pi)$, while the error
contracts at $q_{t_k}/\kappa_{t_k}(s_{k-1})$.

\paragraph{Variance error from the learned step.}
With $\widetilde m\ne m$ and $\eta\equiv0$, the model has learned
$\mathcal N(0,\widetilde m^{-1}\eye)$: the right mean and the wrong variance.
This is the extreme of a learned step that errs only in the variance, and it
holds exactly when $\pi$ is symmetric and the model respects the symmetry,
since an odd affine score has no offset. The score is nonetheless wrong,
$e_r=+\infty$, and no finite
$E_t$ satisfies \cref{ass:oneshot-approximation}, since
$W_2\big(\CMstep t(\nu_s),\CMstepA t(\nu_s)\big)
 =\sqrt d\,\abs{L_t-\widetilde L_t}(\alpha_t^2s^2+\sigma_t^2)^{\nicefrac{1}{2}}\to\infty$
as $s\to\infty$. \Cref{ass:approx-oneshot,ass:approx-denoise} still hold,
with $\widetilde L_t$ and $E^\pi_t=\Delta_t$, so
\cref{prop:learned-recursion} applies where \cref{thm:main} does not: a CM
that learned a slightly wrong variance, a first-order description of the
models of \cref{sec:experiments}, admits no uniform error bound, however
accurate at the target. Its certificate overstates the inherited
variance error by $\kappa_t$ and keeps a gap at the plateau, where \eqref{eq:gauss-verif-plateau-ratio} equals
$\Delta_t/[(1-\widetilde q_t)\delta_t]$ (\cref{fig:gauss-verif-misses}).

\paragraph{Both errors.}
The mean error contracts at $\widetilde q_{t_k}$ and the inherited variance
error at $\widetilde q_{t_k}/\kappa_{t_k}$, so the variance component
dominates early and the mean component late: the certificate is loose, then
tight, and the crossover comes earlier as $\omega_t$ grows and later as the
initial variance error grows (\cref{fig:gauss-verif-misses}). At the
plateau the mean error dilutes the gap, \eqref{eq:gauss-verif-plateau-ratio}
decreasing from its value at $\omega_t=0$ to one. A tight certificate thus indicates that the dominant error is in the
mean, not that the denoiser is accurate.

The same balance governs the choice of the next level.
\Cref{fig:gauss-verif-schedule} (left) takes one step at a large fixed level $t_1$
and then sweeps the level $t_2$ of a second step, for an initialization and
a learned step that are both too wide, so that no cancellation occurs.
Without an error in the mean, a larger $t_2$ contracts the remaining variance
error more strongly, and the exact error and the certificate both favour the
largest level. An error in the mean makes each step add a bias $\omega_{t_2}$ that
grows with $t_2$, so the exact error has an interior optimum. The
certificate, which overstates the inherited variance error by
$\kappa_{t_2}$, prefers more noise than the optimum, and for a large mean
error it still takes a step when none helps. Its choice nonetheless costs
only a few percent of the error.

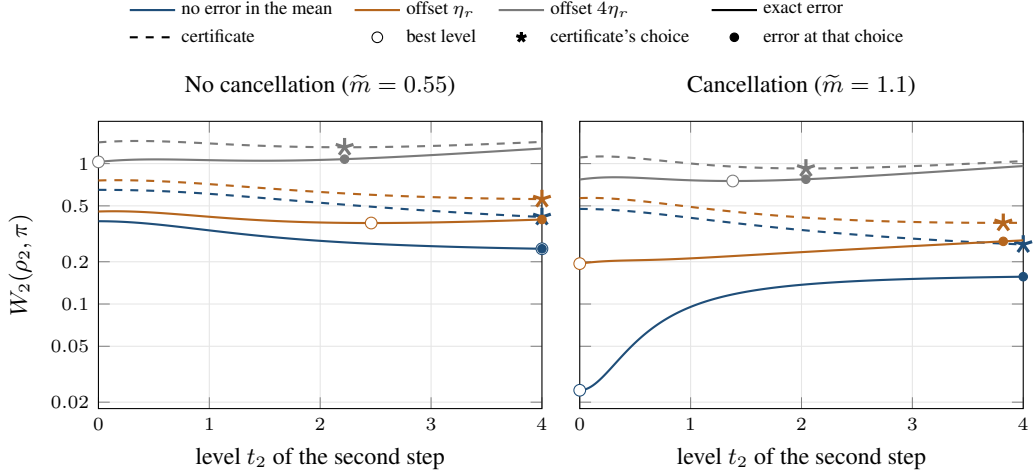
\begin{figure}[tb]
  \centering
  \definecolor{gaussblue}{HTML}{1F4E79}
\definecolor{gaussorange}{HTML}{B5651D}
\definecolor{gaussred}{HTML}{8B1A1A}
\definecolor{gaussgrey}{HTML}{7A7A7A}
\definecolor{gaussbrown}{HTML}{C26A2A}
\begin{tikzpicture}
\begin{groupplot}[scale only axis, clip mode=individual, grid=major, grid style={gray!20}, yticklabel style={/pgf/number format/1000 sep={}}, tick label style={font=\scriptsize}, label style={font=\small}, title style={font=\small}, legend style={font=\scriptsize, draw=none, fill=none}, legend cell align=left, group style={group size=2 by 1, horizontal sep=0.5cm,
    yticklabels at=edge left},
  width=0.42\linewidth, height=3.8cm, xmin=0, xmax=4,
  ymode=log, ymin=0.018, ymax=2, log ticks with fixed point,
  ytick={0.02,0.05,0.1,0.2,0.5,1}, xlabel={level $t_2$ of the second step}]
\nextgroupplot[title={No cancellation ($\widetilde m=0.55$)},
  ylabel={$W_2(\rho_2,\pi)$},
  legend to name=gaussschedlegend, legend columns=4,
  legend style={/tikz/every even column/.append style={column sep=0.6em}}]
\addplot[gaussblue, thick, forget plot] table[x=t2, y=exact] {Figures/gaussian/schedule-nocancel-0.dat};
\addplot[gaussblue, thick, dashed, forget plot] table[x=t2, y=cert] {Figures/gaussian/schedule-nocancel-0.dat};
\addplot[gaussblue, only marks, mark=*, mark size=2.2pt, mark options={fill=white}, forget plot] coordinates {(4,0.24737)};
\addplot[gaussblue, only marks, mark=star, mark size=3.5pt, very thick, forget plot] coordinates {(4,0.416374)};
\addplot[gaussblue, only marks, mark=*, mark size=1.6pt, forget plot] coordinates {(4,0.24737)};
\addplot[gaussorange, thick, forget plot] table[x=t2, y=exact] {Figures/gaussian/schedule-nocancel-1.dat};
\addplot[gaussorange, thick, dashed, forget plot] table[x=t2, y=cert] {Figures/gaussian/schedule-nocancel-1.dat};
\addplot[gaussorange, only marks, mark=*, mark size=2.2pt, mark options={fill=white}, forget plot] coordinates {(2.46,0.377449)};
\addplot[gaussorange, only marks, mark=star, mark size=3.5pt, very thick, forget plot] coordinates {(4,0.55876)};
\addplot[gaussorange, only marks, mark=*, mark size=1.6pt, forget plot] coordinates {(4,0.399849)};
\addplot[gaussgrey, thick, forget plot] table[x=t2, y=exact] {Figures/gaussian/schedule-nocancel-4.dat};
\addplot[gaussgrey, thick, dashed, forget plot] table[x=t2, y=cert] {Figures/gaussian/schedule-nocancel-4.dat};
\addplot[gaussgrey, only marks, mark=*, mark size=2.2pt, mark options={fill=white}, forget plot] coordinates {(0,1.02865)};
\addplot[gaussgrey, only marks, mark=star, mark size=3.5pt, very thick, forget plot] coordinates {(2.22,1.30737)};
\addplot[gaussgrey, only marks, mark=*, mark size=1.6pt, forget plot] coordinates {(2.22,1.07559)};
\addlegendimage{gaussblue, thick}\addlegendentry{no error in the mean}
\addlegendimage{gaussorange, thick}\addlegendentry{offset $\eta_r$}
\addlegendimage{gaussgrey, thick}\addlegendentry{offset $4\eta_r$}
\addlegendimage{black, thick}\addlegendentry{exact error}
\addlegendimage{black, thick, dashed}\addlegendentry{certificate}
\addlegendimage{only marks, mark=*, mark options={fill=white}}\addlegendentry{best level}
\addlegendimage{only marks, mark=star, very thick}\addlegendentry{certificate's choice}
\addlegendimage{only marks, mark=*, mark size=1.6pt}\addlegendentry{error at that choice}
\nextgroupplot[title={Cancellation ($\widetilde m=1.1$)}]
\addplot[gaussblue, thick, forget plot] table[x=t2, y=exact] {Figures/gaussian/schedule-cancel-0.dat};
\addplot[gaussblue, thick, dashed, forget plot] table[x=t2, y=cert] {Figures/gaussian/schedule-cancel-0.dat};
\addplot[gaussblue, only marks, mark=*, mark size=2.2pt, mark options={fill=white}, forget plot] coordinates {(0,0.024308)};
\addplot[gaussblue, only marks, mark=star, mark size=3.5pt, very thick, forget plot] coordinates {(4,0.265249)};
\addplot[gaussblue, only marks, mark=*, mark size=1.6pt, forget plot] coordinates {(4,0.15651)};
\addplot[gaussorange, thick, forget plot] table[x=t2, y=exact] {Figures/gaussian/schedule-cancel-1.dat};
\addplot[gaussorange, thick, dashed, forget plot] table[x=t2, y=cert] {Figures/gaussian/schedule-cancel-1.dat};
\addplot[gaussorange, only marks, mark=*, mark size=2.2pt, mark options={fill=white}, forget plot] coordinates {(0,0.193946)};
\addplot[gaussorange, only marks, mark=star, mark size=3.5pt, very thick, forget plot] coordinates {(3.82,0.378454)};
\addplot[gaussorange, only marks, mark=*, mark size=1.6pt, forget plot] coordinates {(3.82,0.279349)};
\addplot[gaussgrey, thick, forget plot] table[x=t2, y=exact] {Figures/gaussian/schedule-cancel-4.dat};
\addplot[gaussgrey, thick, dashed, forget plot] table[x=t2, y=cert] {Figures/gaussian/schedule-cancel-4.dat};
\addplot[gaussgrey, only marks, mark=*, mark size=2.2pt, mark options={fill=white}, forget plot] coordinates {(1.38,0.751702)};
\addplot[gaussgrey, only marks, mark=star, mark size=3.5pt, very thick, forget plot] coordinates {(2.04,0.923409)};
\addplot[gaussgrey, only marks, mark=*, mark size=1.6pt, forget plot] coordinates {(2.04,0.773712)};
\end{groupplot}
\node[anchor=south] at (current bounding box.north) {\pgfplotslegendfromname{gaussschedlegend}};
\end{tikzpicture}
  \caption{Choosing the level of a second step, in VE coordinates with $m=0.8$
  and $d=1$: a first step at $t_1=4$ from $\mathcal N(0,2.5^2\eye)$, then a
  second at level $t_2$, for no error in the mean and for the offsets
  $\eta_r$ and $4\eta_r$, with $\eta_r$ as in
  \cref{fig:gauss-verif-translation}, on a logarithmic scale. Exact error
  solid, certificate dashed; circles mark the best level ($t_2=0$: no second
  step), stars the certificate's choice and dots the exact error at that
  choice. Left: the learned variance
  is too large ($\widetilde m=0.55$), like the initialization, so no
  cancellation occurs. Right: it is too small ($\widetilde m=1.1$), so the
  learned step makes the law too narrow while the initialization is too
  wide, and the two variance errors cancel.}
  \label{fig:gauss-verif-schedule}
\end{figure}

\paragraph{Cancellation.}
When the inherited and created variance errors have opposite directions, the
exact error in \eqref{eq:gauss-verif-onestep} adds them with their signs
while the certificate adds their sizes, and the certificate can exceed the
error by an arbitrary factor: if $\widetilde m>m$ and $\eta\equiv0$, the step
makes laws too narrow, and the too wide law with $S_t(s)=m^{-\nicefrac{1}{2}}$ is mapped
exactly onto $\pi$ while its certificate is positive.
This needs a variance error entering the step in the direction opposite
to the one the step creates. A model whose variance error keeps one
direction across levels produces such an error only through the
initialization, which a large first level largely erases.
The right panel of \cref{fig:gauss-verif-schedule} repeats the
comparison with a learned variance that is too small, against the same too
wide initialization. After the first step the two variance errors nearly
cancel, so the law is closer to $\pi$ than its certificate says, and the
certificate recommends a second step that makes the error worse. The error in
the mean, certified exactly, limits the damage: with no error in the mean
the law after one step is almost exactly $\pi$ and the recommended step
multiplies the error several times; with the offset $\eta_r$ it still
increases the error; with $4\eta_r$ a second step helps, and following the
certificate costs little.

\paragraph{Scope.}
For $\pi=\mathcal N(0,\Sigma)$ the profiles hold with
$m=\lambda_{\max}(\Sigma)^{-1}$ and the recursions extend coordinatewise; the
certificate is then driven by the flattest direction while the error averages
over all of them, a further, dimension-dependent slack. The model does not
show that these mechanisms are at work in
\cref{sec:experiments}, which would require measuring the error across step
counts for a trained model.

\section{Validating the one-step bound on SD~1.5}\label{app:one-step-bound}
We combine the estimated profiles $\widehat q_t$ and $\widehat E^\pi_t$ (\cref{sec:estimators}) of the SD~1.5 LCM into the one-step $W_2$ bound that underpins our recursion, and compare this bound with the empirical $W_2$ error measured by an optimal transport estimator. More precisely, starting from $\rho_0 = \mathcal{N}(0, I)$, we apply a single noising--denoising step at each VE level $t$ separately, and compute the empirical $W_2$ distance between the resulting latent samples and latent samples from a target dataset (different from the one used to calibrate $\widehat q_t$ and $\widehat E^\pi_t$), using the exact optimal assignment between the two point clouds. We compare this distance with the plug-in one-step bound
\begin{equation}\label{eq:empirical-one-step-bound}
    \widehat{B}_1(t) := \widehat q_t\,\widehat{D}_0 + \widehat{E}^\pi_t,
    \qquad
    \widehat{D}_0 \approx W_2(\rho_0, \pi).
\end{equation}

\begin{figure}[t]%
    \centering
    \begin{tikzpicture}
\begin{axis}[
  theoryplot,
  height=2.8cm,
  xmode=log,
  xmin=0.13, xmax=16.5,
  ymin=118, ymax=194,
  ytick={120,130,140,150,160,170,180,190},
  xlabel={VE time $t=\sigma$},
  ylabel={OT vs bound},
  legend style={at={(0.98,0.97)}, anchor=north east},
]
\addplot[gaussblue, thick]
  table[x=t,y=empirical] {Figures/theory/data/ve_one_step_bound.dat};
\addlegendentry{Measured empirical OT $W_2$}

\addplot[gaussorange, thick, dashed]
  table[x=t,y=bound] {Figures/theory/data/ve_one_step_bound.dat};
\addlegendentry{Recursive bound}

\addplot[gaussred, only marks, mark=*, mark size=2.2pt]
  coordinates {(8.6399,121.232)};
\addlegendentry{Minimum (empirical OT $W_2$)}

\addplot[gaussorange, only marks, mark=*, mark size=2.2pt]
  coordinates {(8.6399,129.094)};
\addlegendentry{Minimum (recursive bound)}
\end{axis}
\end{tikzpicture}%
    \caption{One-step SD~1.5 experiment. The measured empirical W2 error is
    compared with the plug-in recursive estimate
    \mbox{$\widehat q_t\widehat D_0+\widehat E_t^\pi$}.}
    \label{fig:one-step-bound}
\end{figure}
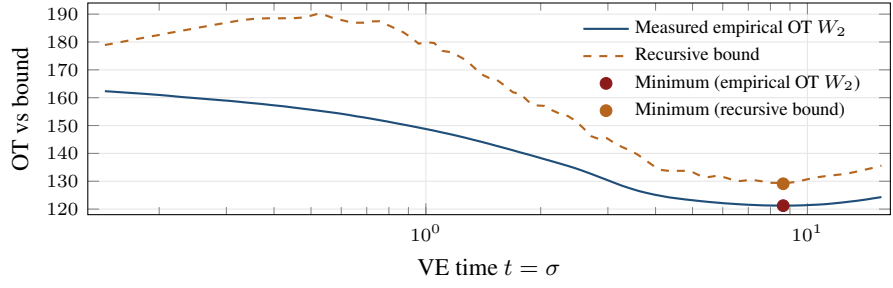

As shown in \cref{fig:one-step-bound}, $\widehat{B}_1(t)$ lies above the observed $W_2$ error over the full range of VE times and follows its variation with $t$. Since $\widehat{B}_1(t)$ is assembled from independently estimated profiles rather than fitted to the measured error, its proximity to that error indicates that the bound is reasonably tight. Interestingly, the gap is smallest near the optimal value of $t$, which is the regime that matters for schedule design.

Unlike in the Gaussian case, we do not report multistep bounds for the considered CMs. Already for two
steps, the bound drops to the level of the statistical bias of the empirical
$W_2$ estimator, so it cannot be meaningfully compared with the measured error.
Resolving this would require vastly more samples, since empirical $W_2$ estimates converge at rate that is notoriously slow in high dimensions.

\end{document}